\documentclass{article}
\usepackage{times}
\usepackage{xcolor}
\usepackage{soul}
\usepackage[small]{caption}
\usepackage{subfigure}
\usepackage{graphicx} 
\usepackage{todonotes}

\usepackage[letterpaper]{geometry}
\usepackage[parfill]{parskip}
\PassOptionsToPackage{numbers, compress}{natbib}
\usepackage[numbers, compress]{natbib}
\usepackage{xr-hyper,refcount}

\usepackage[utf8]{inputenc} 
\usepackage[T1]{fontenc}    
\usepackage{hyperref}       
\usepackage{url}            
\usepackage{booktabs}       
\usepackage{amsfonts}       
\usepackage{nicefrac}       
\usepackage{microtype}      
\usepackage{graphicx}
\usepackage{subfigure}
\usepackage{booktabs} 
\usepackage{pstool}

\usepackage{wrapfig}

\usepackage{color,soul}
\usepackage{microtype}
\usepackage{graphicx}
\usepackage{subfigure}
\usepackage{booktabs} 
\usepackage{hyperref}
\usepackage{algorithm}
\usepackage[noend]{algorithmic}
\usepackage{xspace}
\usepackage{amsmath,amsfonts,epsf,psfrag,amssymb,bbm,bm}
\usepackage{setspace}
\usepackage{makecell}

\newcommand{\wt}[1]{\widetilde{#1}}
\newcommand{\wh}[1]{\widehat{#1}}

\newcommand{\DQN}{\textsc{\small{DQN}}\xspace}

\newcommand{\DDQN}{\textsc{\small{DDQN}}\xspace}

\newcommand{\GAN}{\textsc{\small{GAN}}\xspace}
\newcommand{\piroll}{\pi_{r}}
\newcommand{\GDM}{\textsc{\small{GDM}}\xspace}
\newcommand{\RP}{\textsc{\small{RP}}\xspace}

\newcommand{\GATS}{\textsc{\small{GATS}}\xspace}

\newcommand{\A}{\mathcal A}
\newcommand{\X}{\mathcal X}
\renewcommand{\L}{\mathcal L}
\newcommand{\E}{\mathbb E}
\newcommand{\Prob}{\mathbb P}

\newcommand{\N}{\mathcal N}

\newtheorem{lemma}{Lemma}

\newtheorem{proposition}{Proposition}

\newtheorem{proof}{Proof}

\title{Surprising Negative Results for Generative
Adversarial Tree Search}

\begin{document}
\author{
\hspace{-1.5cm}
  Kamyar Azizzadenesheli$^1$, Brandon Yang$^{2}$, Weitang Liu$^3$, Zachary C. Lipton$^{4}$, Animashree Anandkumar$^1$\\
  \hspace{-1.7cm} Caltech$^1$,
  Stanford University$^2$,
  UC Davis$^3$,
Carnegie Mellon University$^4$
}

\maketitle

\begin{abstract}
While many recent advances in deep reinforcement learning (RL) rely on model-free methods, model-based approaches remain an alluring prospect for their potential to exploit unsupervised data to learn environment model. In this work, we provide an extensive study on the design of deep generative models for RL environments and propose a sample efficient and robust method to learn the model of Atari environments. We deploy this model and propose generative adversarial tree search (\GATS) a deep RL algorithm that learns the environment model and implements Monte Carlo tree search (MCTS) on the learned model for planning. While MCTS on the learned model is computationally expensive, similar to AlphaGo, \GATS follows depth limited MCTS. \GATS employs deep Q network (DQN) and learns a Q-function to assign values to the leaves of the tree in MCTS. We theoretical analyze \GATS  vis-a-vis the bias-variance trade-off and show \GATS is able to mitigate the worst-case error in the Q-estimate. While we were expecting \GATS to enjoy a better sample complexity and faster converges to better policies, surprisingly, \GATS fails to outperform DQN. We provide a study on which we show why depth limited MCTS fails to perform desirably. 


\end{abstract}


\section{Introduction}
\label{sec:intro}
The earliest and best-publicized applications of deep RL involve Atari games~\citep{mnih2015human} and the board game of Go~\citep{silver2016mastering} which are simulated environments and experiences are inexpensive.  In such scenarios, model-free deep RL methods have been used to design suitable policies. But these approaches mainly suffer from high sample complexity and are known to be biased~\citep{thrun1993issues,antos2008learning,van2016deep}, making them less appropriate when the experiences are expensive. To mitigate the effect of the bias, one can deploy  MCTS methods~\citep{kearns2002sparse,kocsis2006bandit} to enhance the quality of policies by rolling out on the simulated environment. But this approach, for problems with a long horizon, e.g., Go, becomes computationally expensive. As a remedy, Alpha Go~\citep{silver2016mastering} propose to combine model free deep RL methods with model-based MCTS. They employ a depth-limited MCTS on the Go emulator and learn a Q-function to assign values to the leaf nodes. Despite advances in the Go game~\citep{silver2016mastering}, this approach still does not fully address the sample complexity issue. Moreover, in real-world applications, such as robotics~\citep{levine2016end} and dialogue systems~\citep{lipton2016combating}, not only collecting experiences takes considerable effort, but also there is no such simulator for these environments.

Recently, generative adversarial networks (GANs)~\citep{goodfellow2014generative} have emerged as a prominent tool for synthesizing realistic-seeming data, especially for high-dimensional domains, e.g., images. Unlike previous approaches to image generation, which typically produced blurry images due to optimizing on L1 or L2 loss, GANs produces crisp images. GANs have been extended to conditional generation, 
e.g., generating images conditioned on labels~\citep{mirza2014conditional,odena2016conditional} or next frames of a video given a context window~\citep{mathieu2015deep}. Recently, the \textsc{pix2pix} approach propose to use U-Net architecture and has demonstrated impressive results on a range of image-to-image translation tasks~\citep{isola2017image}.

In this study, we use Atari games in Arcade Learning Environment (ALE)~\citep{bellemare2013arcade} as our testbed to design sample efficient deep RL algorithms. We propose generative adversarial tree search (GATS), a Deep RL algorithm that learns the model of the environment and performs depth-limited MCTS on the learned model for planning. \GATS consists of three main components: 1) generative dynamics model (\GDM), a deep generative model that leverages \textsc{pix2pix} GANs and efficiently learns the dynamics of the environments. \GDM uses Wasserstein distance~\citep{arjovsky2017wasserstein} as the learning metric and deploys spectral normalization technique~\citep{miyato2018spectral} to stabilize the training. Condition on a state and a sequence of actions, \GDM produces visually crisp successor frames that agree closely with the real frames of games. 2) reward predictor (\RP), a classifiers that predict future rewards, clipped rewards of $[-1,0,1]$ in ALE, 3) a value based deep RL component to assign value to the leaf nodes of the depth limited trees in MCTS,~Fig.~\ref{fig:mcts}. For this purpose, we use \DQN~\citep{mnih2015human} and \DDQN~\citep{van2016deep}, but any other value-based deep RL approach is also suitable.

\textbf{\GDM:} Designing an efficient and robust \GDM model for RL task, in particular for Atari games is significantly challenging. The recent studies on GANs are manly dedicated to generating samples from a fixed distribution, while neither fast convergence, fast adaption, nor continual learning is considered. Even in the case of fixed distributions, GANs are known to be unstable and hard to train. RL problems are entirely on the opposite side of the hardness spectra. For RL problems, we need a \GDM with a proper model capacity that statistically adapts quickly to the distribution changes due to policy updates and exploration. Followed by the online nature of RL, we require a \GDM that computationally converges fast in the presence of changes in the data distribution, e.g. learning a new skill. More importantly, we need a \GDM that continually learns the environment dynamics without diverging or becoming unstable, even once, therefore being robust. These are the critical requirements for a useful \GDM. Despite the high computation cost of such design study, we thoroughly and extensively study various image translation methods, various GAN-based losses, architectures, in both feed-forward and recurrent neural networks, to design a \GDM, suitable for RL tasks. As a result of this study, we propose a \GDM architecture and learning procedure that reasonably satisfy all the mentioned requirements. We test the performance of \GDM visually, with L1 loss, L2 loss, and also by testing a Q function on generated and real samples.
%
%

\textbf{Theoretical analysis:}
We analyze the components of error in the estimation of the expected return in \GATS, including the bias and variance in the Q-function. We empirical study the error in the Q estimate of \DQN/\DDQN. Since \GATS deploys the Q function in the leaf nodes of the MCTS, we show that the errors in the Q-estimation disappear exponentially fast as the depth of MCTS grows. 

\textbf{Domain change results:} 
In order to thoroughly test the \GDM and \RP, we developed a new OpenAI~\citep{1606.01540} gym-like interface for the latest \emph{Arcade Learning Environment (ALE)}~\citep{machado2017revisiting} that supports different modes and difficulties of Atari games. We documented and open-sourced this package along with the codes for \GDM, \RP, and \GATS. We study the sample complexity of \GDM and \RP in adapting and transfer from one domain (mode and difficulty in ALE) to another domain. We show that the  \GDM and \RP  surprisingly are even robust to mode changes, including some that destroy the \DQN policy. They adapt to the new domain with order of thousands samples, while the Q-network requires significantly more, order of millions. 

Our initial empirical studies are conducted on Pong. The comparably low computation cost of Pong allows an extensive investigation of \GATS. We study MCTS with depths at most 5, even for Pong, each run $5M$ time steps requires at least five weeks of GPU time. This is evidence of the massive computational complexity of \GATS. We extend our study to four more popular Atari games.

\textbf{Surprising negative results:} 
Despite learning near-perfect environment model (\GDM and \RP) that achieves accuracy exceeding our expectations, \GATS is unable to outperform its base model-free model \DQN/\DDQN on any Atari game besides Pong for which we deployed an extra substantial parameter tuning. To boost up \GATS, we make an extensive and costly study with many different hyper-parameters and learning strategies, including several that uses the generated frames to train the Q-model, e.g., DynaQ~\citet{sutton1990integrated}. We also develop various exploration strategies, including optimism-based strategy for which We use the errors in the \GDM to navigate the exploration. The negative result persisted across all these innovations, and none of our developments helped to provide a modest improvement in \GATS performance.
Our initial hypothesis was that \GATS helps to improve the performance of deep RL methods, but after the extensive empirical study and rigorous test of this hypothesis, we observe that \GATS does not pass this test. We put forth a new hypothesis for why \GATS, despite the outstanding performance of its components along with the theoretical advantages, might fail in short rollouts. In fact, we show that \GATS locally keeps the agent away from adverse events without letting the agent learn from them, resulting in faulty Q estimation. \citet{holland2018effect}, in fact, observe that it might require to make the depth up to the effective horizon in order to draw an improvement.



To make our study concrete, we create a new environment, \textit{Goldfish and gold bucket}. To exclude the effect of model estimation error, we provide \GATS with the true model for the MCTS. We show that as long as the depth of the MCTS is not close to the full effective horizon of the game, \GATS does not outperform \DQN.  This study results in a conclusion that combining MCTS with model-free Deep RL, as long as the depth of MCTS is not deep enough, might degrade the performance of RL agents. It is important to note that the theoretical justification does not contradict our conclusion. The theoretical analysis guarantees an improvement in the worse case error in Q estimation, but not in the average performance. It also does not state that following \GATS would result in a better Q estimation; in fact, the above argument states that \GATS might worsen the Q estimation.

Consider that all the known successes of MCTS involve tree depth in the hundreds, e.g., $300$  on Atari emulator\citep{guo2014deep}. Such deep MCTS on \GDM requires massive amounts of computation beyond the scale of academic research and the scope of this paper. Considering the broader enthusiasm for both model-based RL and GANs, we believe that this study, despite its failure to advance the leaderboard of deep RL, illuminates several important considerations for future work in combining model-based and model-free reinforcement learning.

\textbf{Structure:} This paper consists of a long study on \GATS and concludes with a set of negative results. We dedicate the main body of this paper to present \GATS and its components. We leave the detail of these our developments to the supplementary section and instead devote the rest to explain the negative results. We hope that the readers find this structure succinct, useful, and beneficial.



\section{Related Work}
\label{sec:related}
Sample efficient exploration-exploitation trade-off is extensively studied in RL literature from Bandits~\citep{auer2003using}, MDPs~\citep{kearns2002near,brafman2003r,asmuth2009bayesian,bartlett2009regal,jaksch2010near}, and Monte Carlo sampling~\citep{kearns2002sparse,kocsis2006bandit} to partially observable and partial monitoring \citep{azizzadenesheli2016reinforcement,bartok2014partial} in low dimensional environment, leaves it open in high dimensional RL~\citep{mnih2015human,abel2016exploratory,azizzadenesheli2016rich}. 
To extend exploration-exploitation efficient methods to high dimensional RL problems, \citet{lipton2016efficient} suggests  variational approximations, \citet{osband2016deep} suggests a bootstrapped-ensemble, \citet{bellemare2016unifying} propose a surrogate to optimism, \citep{fortunato2017noisy} suggest noisy exploration. 
Among these methods, \citet{abbasi-yadkori2011improved,azizzadenesheli2018efficient} provide insight in high dimensional linear RL problems. While most of the deep RL methods are based on model-free approaches, the model-based deep RL has been less studied.

Model based deep RL approaches requires building a notion of model for the reasoning. 
%
%
Recently, conditional video predictions based on L1 and L2 losses~\citep{oh2015action} have been deployed by \citet{weber2017imagination} who train a neural network to encodes the generated trajectories into an abstract representation, which then is used as an additional input to the policy model. They validate their methods on Sokoban, a small puzzle  world, and their miniPacman environment. Unlike \GATS, ~\citet{weber2017imagination} does not have explicit planning and roll-out strategies. A similar approach to \GATS is concurrently developed using variational methods and empirically studied on Car Racing and VizDoom \citep{DavidHa}. Other works propose to roll out on a transition model learned on a encoded state representation ~\citep{oh2017value}, and demonstrate modest gains on Atari games. A similar approach also has been studied in robotics~\citep{wahlstrom2015pixels}. In contrast, we learn the model dynamics in the original state/pixel space.

Compared to the previous works, \GATS propose a flexible, modular, and general framework for the study of model-based and model-free RL
with four building blocks: ($i$) value learning  ($ii$) planning ($iii$) reward predictor, and ($iv$) dynamics model. This modularity provides a base for further study and adaptation in future works. For instance, for value learning ($i$): one can use Count-based methods~\citep{bellemare2016unifying}.  For planning ($ii$): one can use upper confidence bound tree search (UCT)~\citep{kocsis2006bandit} or policy gradient methods~\citep{kakade2002natural,schulman2015trust}. 
For the reward model ($iii$): if the reward is continuous, one can deploy regression models. Lastly, for model dynamics ($iv$), one can extend \GDM or choose any other generative model. Furthermore, this work can also be extended to the $\lambda$-return setting,  where a mixture of $n$ steps MCTS and Q is desired. Although \GATS provides an appealing and flexible RL paradigm, it suffers from massive computation cost by modeling the environment. Clearly, such costs are acceptable when experiences in the real world are expensive. Potentially, we could mitigate this bottleneck with parallelization or distilled policy methods~\citep{guo2014deep}.

\section{Generative Adversarial Tree Search}
\label{sec:gats}
\textbf{Bias-Variance Trade-Off:} 
Consider an MDP $M=\langle \X, \A, T, R,\gamma\rangle$, 
with state space $\X$, action space $\A$, transition kernel $T$, reward distribution $R$ with $[0,1]$-bounded mean, and discount factor $0\leq\gamma<1$. 
A policy $\pi$ is a mapping from state to action and $Q_{\pi}(x,a)$ denotes the expected return of action $a$ at state $x$, then following policy $\pi$. 
Following the Bellman equation, the agent can learn the Q function by minimizing the following loss over samples to improve its behavior:
\begin{align}\label{eq:double}
\left(Q(x,a)-\E_{\pi}\left[r+\gamma Q(x',a')\right|x,a]\right)^2
\end{align}
Since in RL we do not have access internal expectation, we instead minimize the Bellman residual~\citep{schweitzer1985generalized}\citep{lagoudakis2003least,antos2008learning} which has been extensively used in DRL literature~\citep{mnih2015human,van2016deep}: 
\begin{align*}
\!\!\!\E_{\pi}\!\!\left[\!\left(Q(x,a)\!-\!\left(r+\gamma Q(x',a')\right)\right)^2\!\Big|x,a\right]\!\!
=\!\!\left(\!Q(x,a)\!-\!\E_{\pi}\!\!\left[r\!+\!\gamma Q(x'\!,\!a')\Big|x,a\right]\!\right)^2\!\!\!\!+\!\!\textnormal{Var}_\pi\!\!\left(\!r\!+\!\gamma Q(x',a')\Big|x,a\!\right)
\end{align*}
%
This is the sum of Eq.~\ref{eq:double} 
and an additional variance term, rustling in a biased estimation of Q function. \DQN deploys a target network to slightly mitigate this bias~\citep{van2016deep} and minimizes:
\begin{align}\label{eq:reg}
\hspace*{-0.3cm}\L(Q,Q^{\text{target}}) \!=\! \E_{\pi}\!\left[\left(Q(x,a)\!-r\!-\!\gamma Q^{target}(x',a')\right)^2\right]\!\!\!\!
\end{align}
In addition to this bias, there are statistical biases due to limited network capacity, optimization, 
model mismatch, and max operator or the choice of $a'$~\citep{thrun1993issues}. 
Let $\wh{\cdot}$ denote an estimate of a given quantity. Define the estimation error in Q function (bias+variance), the reward estimation $\wh{r}$ by \RP, and transition estimation $\wh{T}$ by \GDM as follows: $\forall x,{x}',a\in\X,\A$
\begin{align*}
|Q(x,a)- \wh{Q}(x,a)|\leq e_Q,~\sum_{a}\left|\Big(r(x,a)-\wh{r}({x},a)\Big)\right|\leq e_R,~\sum_{x'}\left|\Big(T(x'|x,a)-\wh{T}({x}'|x,a)\Big)\right|\leq e_T
\end{align*}
For a rollout policy $\piroll$, we compute the expected return using \GDM, \RP and $\wh Q$ as follows:
\begin{align*}
\xi_p(\piroll,x):=\E_{\piroll,\GDM,\RP}\left[\left(\sum_{h=0}^{H-1}\gamma^h\wh{r}_h\right)+\gamma^H\max_a\wh{Q}({x}_H,a)\Big| x\right]
\end{align*}
The agent can compute $\xi_p(\piroll,x)$ with no interaction with the environment.

Define $\xi(\piroll,x) :=\E_{\piroll}\left[\left(\sum_{h=0}^{H-1}\gamma^hr_h\right)+\gamma^H\max_a{Q}({x}_H,a)\Big| x\right]$. 

\begin{proposition}\label{thm:BV}[Model-based Model-free trade-off] 
Building a depth limited MCTS on the learned model, \GDM,\RP, and estimated Q function, $\wh Q$, results in expected return of $\forall x,\piroll$:
\begin{align}
&|\xi_p(\piroll,x)-\xi(\piroll,x)|\leq \frac{1-\gamma^{H}+H\gamma^{H}(1-\gamma)}{(1-\gamma)^2}e_T+\frac{1-\gamma^{H}}{1-\gamma}e_R+\gamma^He_Q,~~\forall x\in\X,\forall\piroll
\end{align}
Proof in the Appendix~\ref{sec:proof} 
\end{proposition}

Therefore, as the depth of the rollout increases, the effect of error in Q function approximation disappears exponentially at the cost of magnifying the error in \GDM and \RP. In the next sections, we show that in fact, the error in  \GDM and \RP are surprisingly low (refer to Appendix~\ref{apx:BV}).
%
%

\textbf{Model:} Algorithm~\ref{alg:bdqn} present \GATS,~Fig.~\ref{fig:mcts}. We parameterize \GDM with $\theta^{\GDM}$, and \RP with $\theta^{\RP}$. We adopt the standard \DQN architecture and game settings~\citep{mnih2015human}. We train a \DQN model as well as \GDM and \RP on samples from replay buffer. To train \GDM, we deploy Wasserstein distance along with L1 and L2 loss, in the presence of spectrally normalized discriminator. For exploration and exploitation trade-off, we deploy $\epsilon$-greedy, the strategy used in \DQN. We also propose a new optimism-based explore/exploit method using the discriminator loss. 
%
%
We empirically observed that computed Wasserstein distance by the discriminator is high for barely visited state-action pairs and low otherwise. Inspired by pseudo-count construction in \citet{ostrovski2017count}, we deploy Wasserstein distance as a notion of count $\tilde{N}(x,a)$. Following the construction of optimism in the face of uncertainty~\citep{jaksch2010near}, we utilize $\tilde{N}(x,a)$ to guide the agent to explore unknown parts of the state space. We define the optimistic (in Wasserstein sense) $\tilde{Q}$ as $\tilde{Q}_{\pi}(x,a)= Q_{\pi}(x,a)+C_{\pi}(x,a)$ where $C_{\pi}(x,a)$ is 
\begin{align}\label{eq:optimism}
c\sqrt{1/\tilde{N}(x,a)}+\gamma\sum_{x'}\wh{T}(x'|x,a)C_{\pi}(x',\pi(x'))
\end{align}
with a scaling factor $c$. We train a separate \DQN to learn the $C$ model, the same way we learn Q. We use $C$ for planning: $\max_\pi\lbrace\wh{\xi}(\pi,x)+C(\pi,x)\rbrace$ instead of $\epsilon$-greedy to guide exploration/exploitation.
\begin{algorithm}[t]
\caption{\GATS(H)}
\begin{algorithmic}[1]
\STATE Initialize parameter sets $\theta$, $\theta^{target}$, $\theta^{\GDM}$,  $\theta^{\RP}, m$, replay buffer and set counter $=~0$
\FOR{episode = 1 to $\inf$}
\FOR{$t=1$ to the end of episode}
    \STATE Compute $a_t$ and sample $\lbrace(x_i, a_i, r_i, x_{i+1})\rbrace_0^m$
    from $MCTS(x_t,H,\theta,\theta^{\GDM},\theta^{\RP})$
	\STATE Set the real $(x_t, a_t, r_t, x_{t+1})$ and generated $\lbrace(x_i, a_i, r_i,x_{i+1})\rbrace_0^m$ to the replay buffer
  \STATE Sample a minibatch of experiences ($x_{\tau}$ , $a_{\tau}$, $r_{\tau}$, $x_{\tau+1}$)
\STATE $y_{\tau} \!\!\gets\! \bigg\{\begin{array}{ll}
 \!\!\!\! r_{\tau} &\!\!\!\!\text{terminal}\\
\!\!\!\! r_\tau + \max_{a'} Q(x_{\tau+1}, a'; \theta^{{target}})&
     \!\!\!\!\text{non-terminal }\\
\end{array}\!\!\!\!\!\!\! $ 
\STATE $\theta \gets \theta - \eta \cdot \nabla_{\theta} (y_\tau - Q(x_{\tau}, a_{\tau}; \theta))^2$ 
\STATE Update \GDM, and \RP
   
\ENDFOR
\ENDFOR
\end{algorithmic}
\label{alg:bdqn}
\end{algorithm}


\section{Experiments}\label{sec:experiments}
We study the performance of \GATS with depth of at most five on five Atari games; Pong, Asterix, Breakout, Crazy Climber 
and Freeway.
For the \GDM architecture, Fig.~\ref{fig:gdm_generator}, we build upon the U-Net model of the image-to-image generator originally used in \textsc{pix2pix}~\citep{isola2017image}. The \GDM receives a state, sequence of actions, and a Gaussian noise vector and generates the next states.\footnote{See the Appendix for the detailed explanation on the architecture and optimization procedure.} The \RP is a simple model with 3 outputs, it receives the current state, action, and the successor state as an input and outputs a label, one for each possible clipped reward $\lbrace-1,0,1\rbrace$. We train \GDM and \RP using prioritized weighted mini-batches of size 128 (more weight on recent samples), and update the two networks every 16 decision steps (4 times less frequently than the Q update).\footnote{For the short lookhead in deterministic environment, we expand the whole tree.}

Our experiments show that with less than $100k$ samples the \GDM learns the environment's dynamics and generalizes well to a test set (\DQN requires around $5M$. We also observe that it adapts quickly even if we change the policy or the difficulty or the mode of the domain. Designing \GDM model is critical and hard since we need a model which learns fast, does not diverge, is robust, and adapts fast. GAN models are known to diverge even in iid sample setting. It become more challenging, when the \GDM is condition on past frames as well as future actions, and needs to predict future given its out generated samples. In order to develop our \GDM, 
we experimented many different model architectures 
for the generator-discriminator, 
as well as different loss functions.

Since the L1 and L2 losses are not good metrics for image generating models, we evaluated the performance of \GDM visually also by applying Q function on generated test sample, Appendix.~\ref{apx:DA}, Fig.~\ref{fig:gdm_curves}. We studied PatchGAN discriminator (patch sizes 1, 16, and 70) and $L1$ loss used in \textsc{pix2pix}~\citep{isola2017image}, finding that this architecture takes approximately $10\times$ more training iterations to learn game dynamics for Pong than the current \GDM. This is likely since learning the game dynamics such as ball position requires the entire frame for the discriminator, more than the patch-based texture loss given by PatchGAN. Previous works propose ACVP and train large models with L2 loss to predict long future trajectories of frames given actions ~\citep{oh2015action}. We empirically studied these methods. While \GDM is much smaller, we observe that it requires significantly fewer iterations to converge to perceptually unidentifiable frames. We also observed significantly lower error for \GDM when a Q function is applied to generated frames from both models. ACVP also struggles to produce meaningful frames in stochastic environments(Appendix~\ref{apx:losses}).

For the choice of the \GAN loss, we first tried the original \GAN-loss~\citep{goodfellow2014generative}, which is based on Jensen–Shannon divergence. With this criterion, not only it is difficult to find the right parameters 
but also not stable enough for non-stationary domains. 
We did the experiments using this loss and trained for Pong while the resulting model was not stable enough for RL tasks. The training loss is sometimes unstable even for a given fixed data set. 
Since, Wasserstein metric provides Wasserstein distance criterion and propose a more general loss in \GAN{}s, we deployed W-GAN~\citep{arjovsky2017wasserstein} for our \GDM. Because W-GAN requires the discriminator to be a bounded Lipschitz function, the authors adopt gradient clipping. Using W-GAN results in an improvement but still not sufficient for RL where fast and stable convergence is required. 

In order to improve the learning stability, parameter robustness, and quality of frames, we also tried a follow-up work on improved-W-GAN~\citep{gulrajani2017improved}, which adds a gradient penalty into the loss of discriminator in order to satisfy the bounded Lipschitzness. Even though it made the \GDM more stable than before, it was still not sufficient due to the huge instability in the loss curve. Finally, we tried spectral normalization~\citep{miyato2018spectral}, a recent technique that not only provides high-quality frames but also converges quickly while the loss function stays smooth. We observed that this model, is stability, robustness to hyperparameter choices, and fast learning thank to spectral normalization combined with W-GAN in the presence of L1 and L2 losses. We deploy this approach in \GDM. 
For the training, we trained \GDM to predict the next frame and generate future frames given its own generated frames. We also trained it to predict multiple future frames. Moreover, we studied both feed-forward and recurrent version of it. Furthermore, we used both one step loss and multi-step loss. After these study, we observed that the feed-forward model which predicts the next state and trained on the loss of the next three frames generalizes to longer rollout and unseen data. More detailed study is left to Appendix. It is worth noting that the code to all these studies are publicly available (Appendix~\ref{apx:GDM}).

Fig.~\ref{fig:exp} shows the effectiveness of \GDM and how accurate it can generate next $9$ frames just conditioning on the previous $4$ frames and a sequnce of actions. We train \GDM using 100,000 frames and a 3-step loss, and evaluate its performance on 8-step roll-outs on unseen 10,000 frames. Moreover, we illustrate the tree constructed by MCTS in Figs.~\ref{fig:gdm_trees},~\ref{fig:gdm_trees1}. We also test a learned Q function on both generated and real frames and observed that the relative deviation is significantly small $(\mathcal{O}(10^{-2}))$. Furthermore, we constructed the tree in MCTS As deep RL methods are data hungry, we can re-use the data generated by \GDM to train the Q-function even more. We also study ways we can incorporate the generated samples by \GDM-\RP to train the Q-function, similar to Dyna-Q~Fig.~\ref{eq:optimism}. 
\paragraph{Shift in Domain} 
We extend our study to the case where we change the game mode. 
In this case, change Pong game mode, the opponent paddle gets halved in size, resulting in much easier game. While we expect a trained \DDQN agent to perform well on this easier game, surprisingly we observe that not only the agent breaks, but also it provides a score of -21, the most negative score possible. While mastering Pong takes $5M$ step from \DDQN, we expected a short fine tuning would be enough for \DDQN to adapt to this new domain. Again surprisingly, we observe that it take 5M times to step for this agent to adapt to this new domain. It means that \DDQN clearly overfit to the original domain. While \DDQN appears unacceptably brittle in this scenario,
\GDM and \RP adapt to the new model dynamics in less than 3k samples, 
which is significantly smaller (Appendix~\ref{apx:DA}). 
As stated before, for this study, we wrote a Gym-style wrapper for the new version of ALE which supports different modes of difficulty levels. This package is publicly available.


\section{Discussion}
\label{sec:discussion}

Our initial hypothesis was that \GATS enhances model-free approaches. With that regard, we extensively studied \GATS to provide an improvement. But, after more than a year of unsuccessful efforts by multiple researchers in pushing for improvement, we reevaluated \GATS from the first principles and realized that this approach, despite near-perfect modeling, surprisingly might not even be capable of offering any improvement. We believe this negative result would help to shape the future study of model-based and model-free RL. In the following, we describe our efforts and conclude with our final hypothesis on the negative result, Table ~\ref{table:reasons}.

\begin{table*}[ht!]
\vspace*{-0.5cm}
  \centering
  \caption{Set of approaches explored to improve \GATS performance}
  \footnotesize
  \begin{tabular}{c|c|c|c}
    \toprule
 Replay Buffer  & Optimizer  & Sampling strategies &  Optimism \\
  \midrule
  \makecell{(i)~Plain DQN\\ (ii)~Dyna-Q}   & \makecell{(i)~Learning rate \\(ii)~Mini-batch size }   & \makecell{(i)~Leaf nodes \\(ii)~Random samples from the tree \\(iii)~Samples by following greedy Q \\~(iv)~Samples by following $\varepsilon$-greedy Q\\~(v) Geometric distribution} &  \makecell{(i)W-loss \\(ii)~exp(W-loss) \\(iii)~L1+L2+W-distance \\(iv)~exp(L1+L2+W-distance)} \\
    \bottomrule
  \end{tabular}
  \label{table:reasons}
\end{table*}

\textbf{Replay Buffer:} 
The agent's decision under \GATS sometimes differs from the decision of the model-free component. To learn a reasonable Q-function, the Q-leaner needs to observe the outcome of its own decision. To address this problem, we tried storing the samples generated in the tree search and use them to further train the Q-model. We studied two scenarios: (i) using plain \DQN with no generated samples and (ii) using Dyna-Q approach and train the Q function also on the generated samples in MCTS. However, these techniques did not improve the performance of \GATS.
\textbf{Optimizer:} Since \GATS, specially in the presence of Dyna-Q is different approach from \DQN, we deploy a further hyper parameter tuning on learning rate and mini-batch size.

\textbf{Sampling strategy:} 
We considered a variety ways to exploit the samples generated in tree search for further learning the Q. (i) Since we use the Q-function at the leaf nodes, we further train the Q-function on generated experience at the leaf nodes. (ii) We randomly sampled generated experience in the tree to further learn the Q. (iii) We choose the generated experience following the greedy action of the Q-learner on the tree, the trajectory that we would have received by following the greedy decision of Q (counterfactual branch). We hypothesized that training the Q-function on its own decisions results in an improvement. (iv) We also considered training on the result of $\varepsilon$-greedy policy instead of the greedy. (v) Finally, since deeper in the tree has experiences with bigger shift the policy, we tried a variety of different geometric distributions to give more weight of sampling to the later part of the tree to see which one was most helpful.

\textbf{Optimism:}
We observed that parts of the state-action space that are novel to the \GDM are often explored less and result in higher discriminator errors. We added (i) the $W$-loss and also (ii) its exponent as a notion of external reward to encourage exploration/exploitation. In (iii) and (iv) We did the same with the summation of different losses, e.g., L1, L2.
%

Despite this extensive and costly study on \GATS, we were not able to show that \GATS,~Fig.~\ref{fig:Asterix-2} benefits besides a limited improvement on Pong (Appendix~\ref{apx:pong},\ref{apx:negav}). 

\begin{figure}[ht!]
\centering
\includegraphics[scale=0.28]{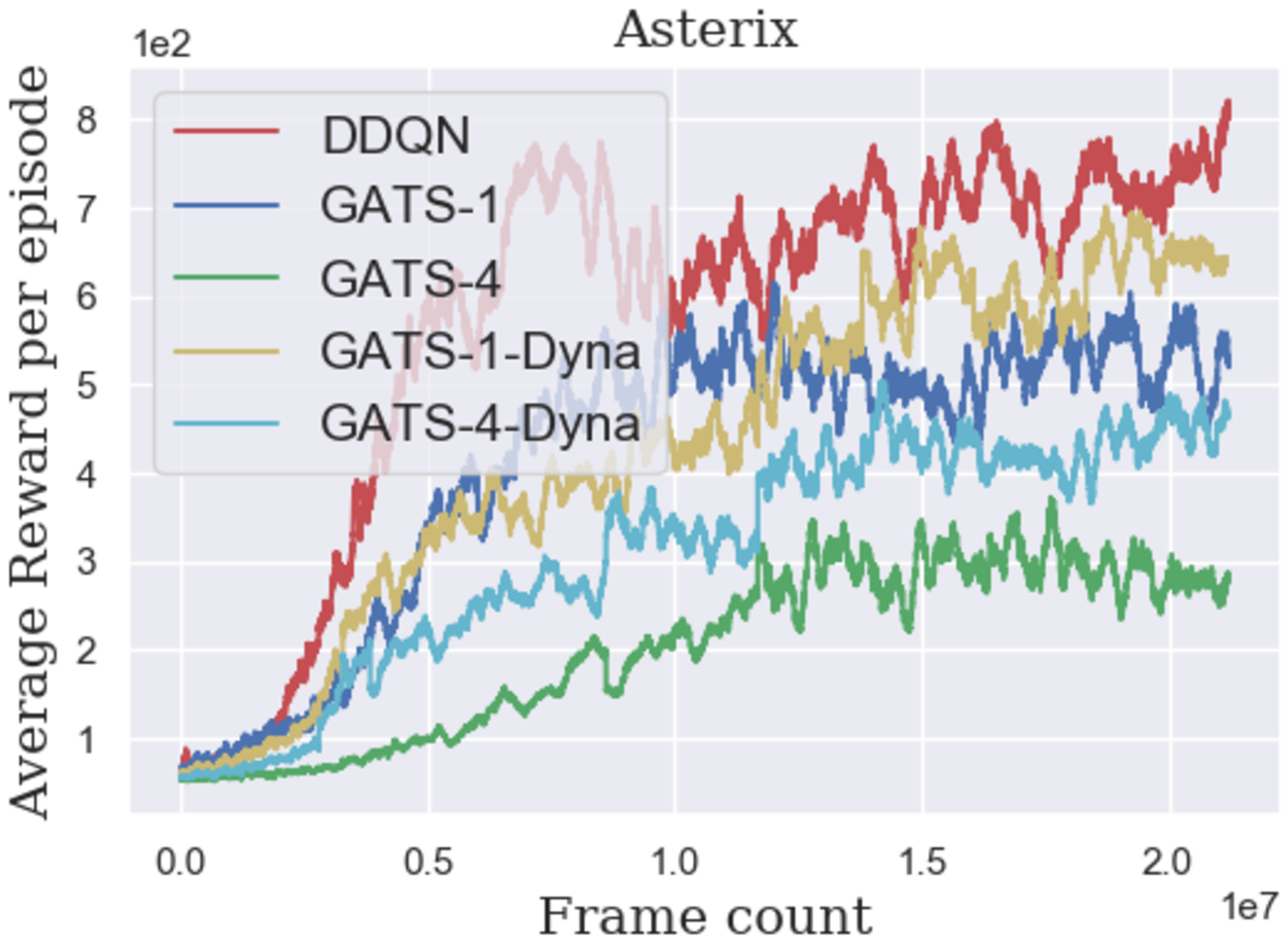}
\includegraphics[scale=0.28]{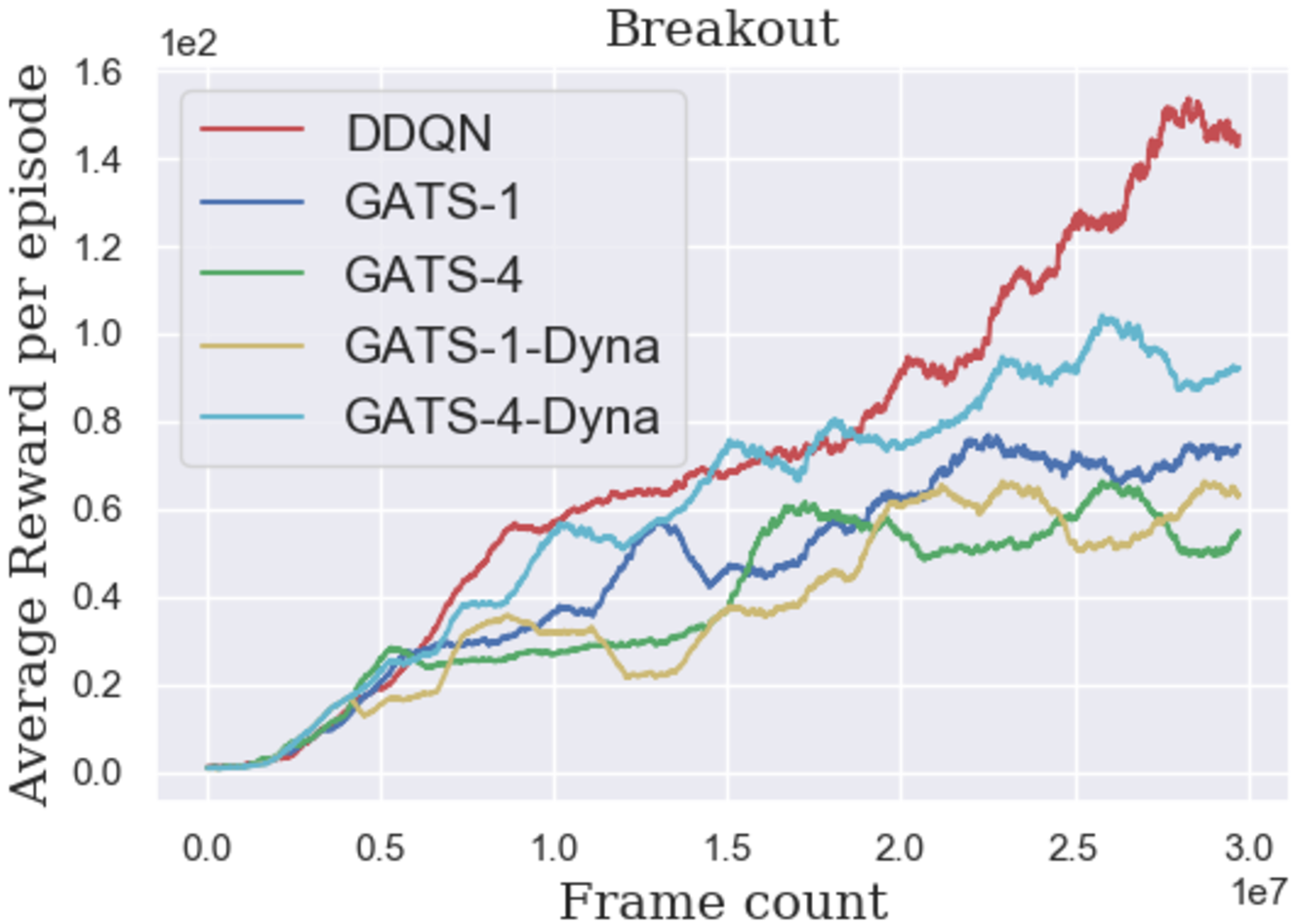}
\caption{\GATS and \GATS+Dyna with 1 and 4 step look-ahead.}
\label{fig:Asterix-2}
\end{figure}

\begin{figure}[ht]
\centering
\begin{minipage}{0.2\textwidth}
\centering
\hspace*{-.7cm}
    \includegraphics[width=1.\textwidth]{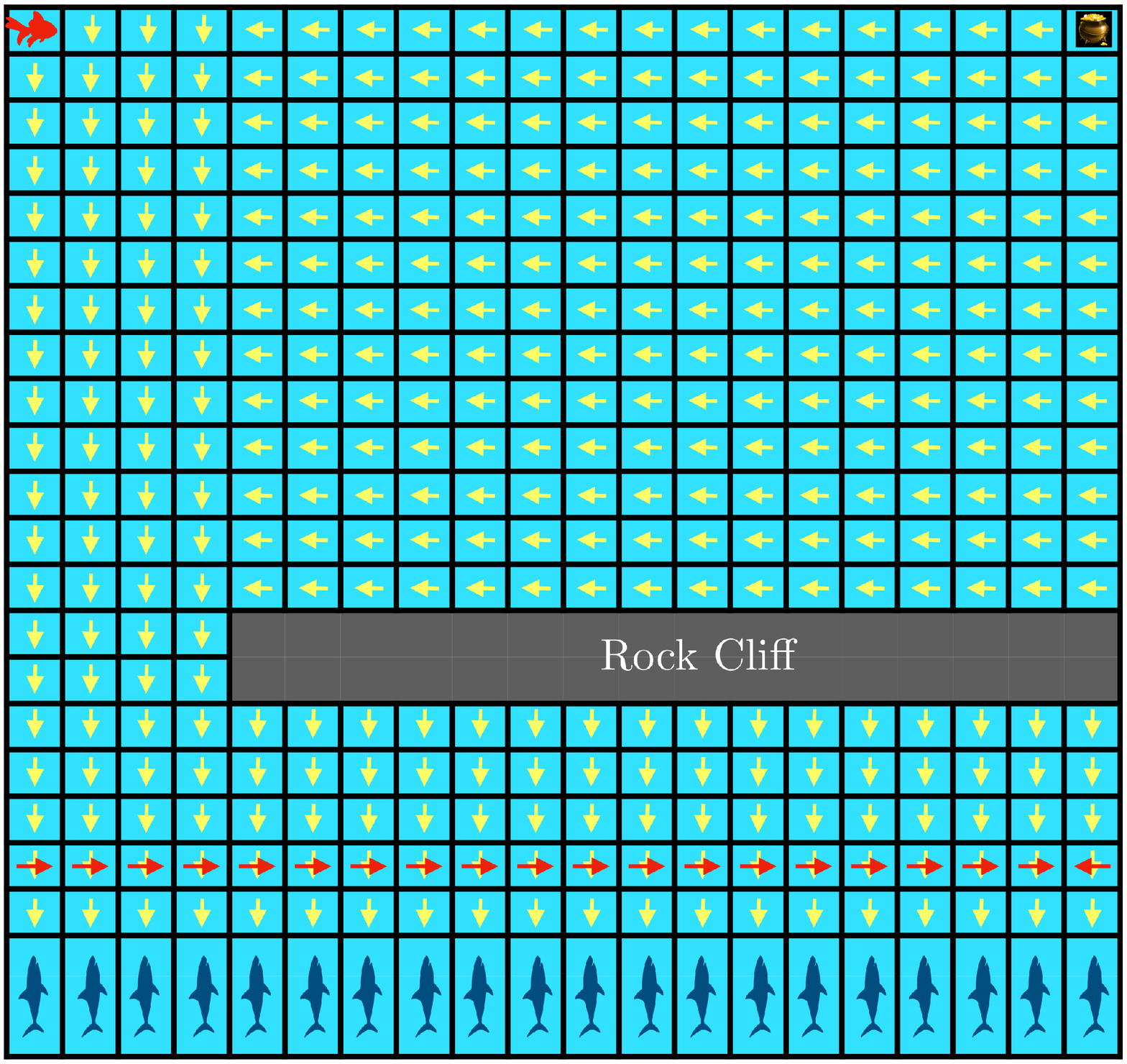}\\
\hspace*{-.7cm}  
(a)
    \end{minipage}
    \begin{minipage}{0.2\textwidth}
    \centering
    \includegraphics[width=1.\textwidth]{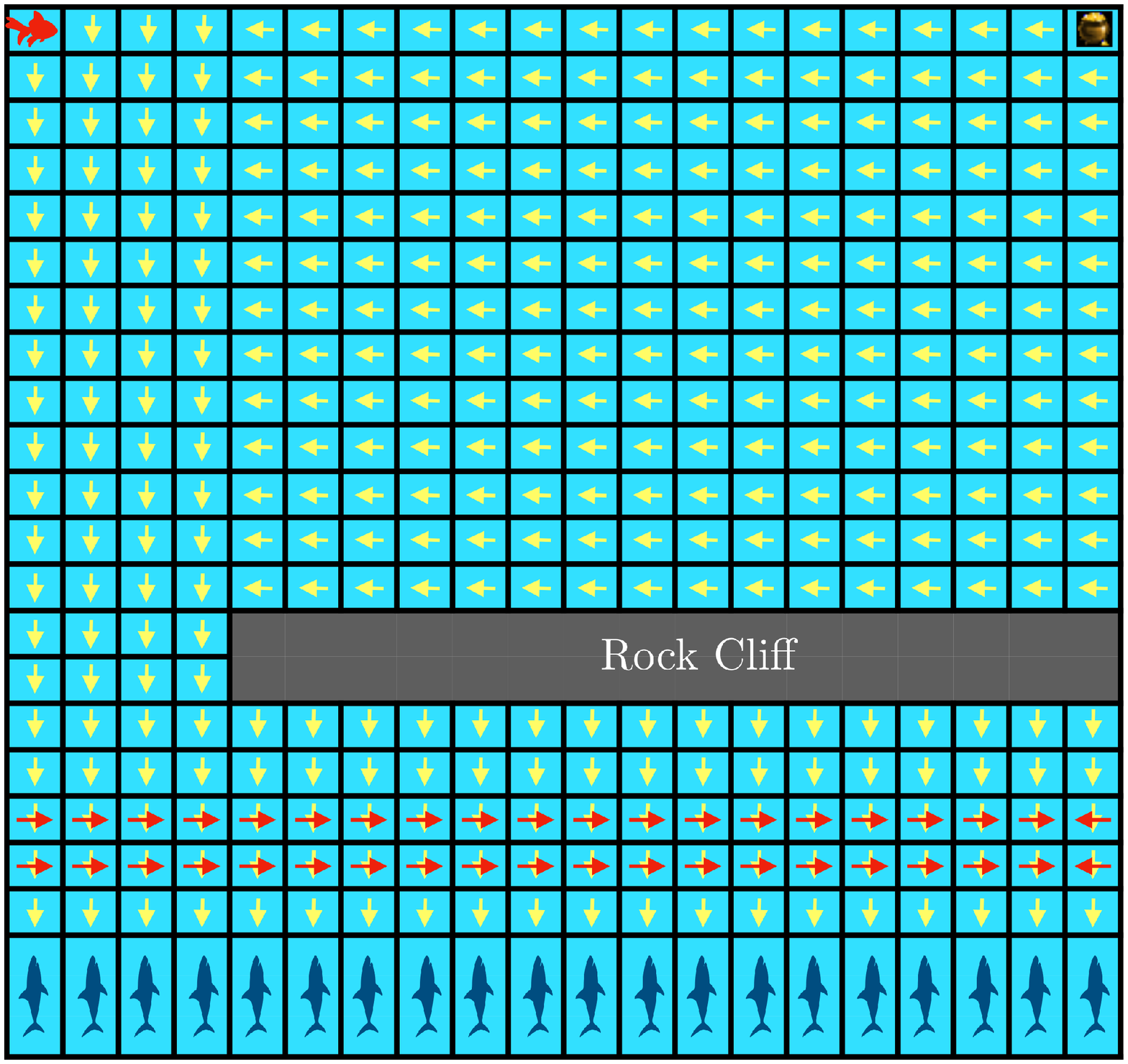}\\
    (b)
    \end{minipage}
    \begin{minipage}{0.2\textwidth}
    \centering
    \includegraphics[width=1.4\textwidth,height=1.\textwidth]{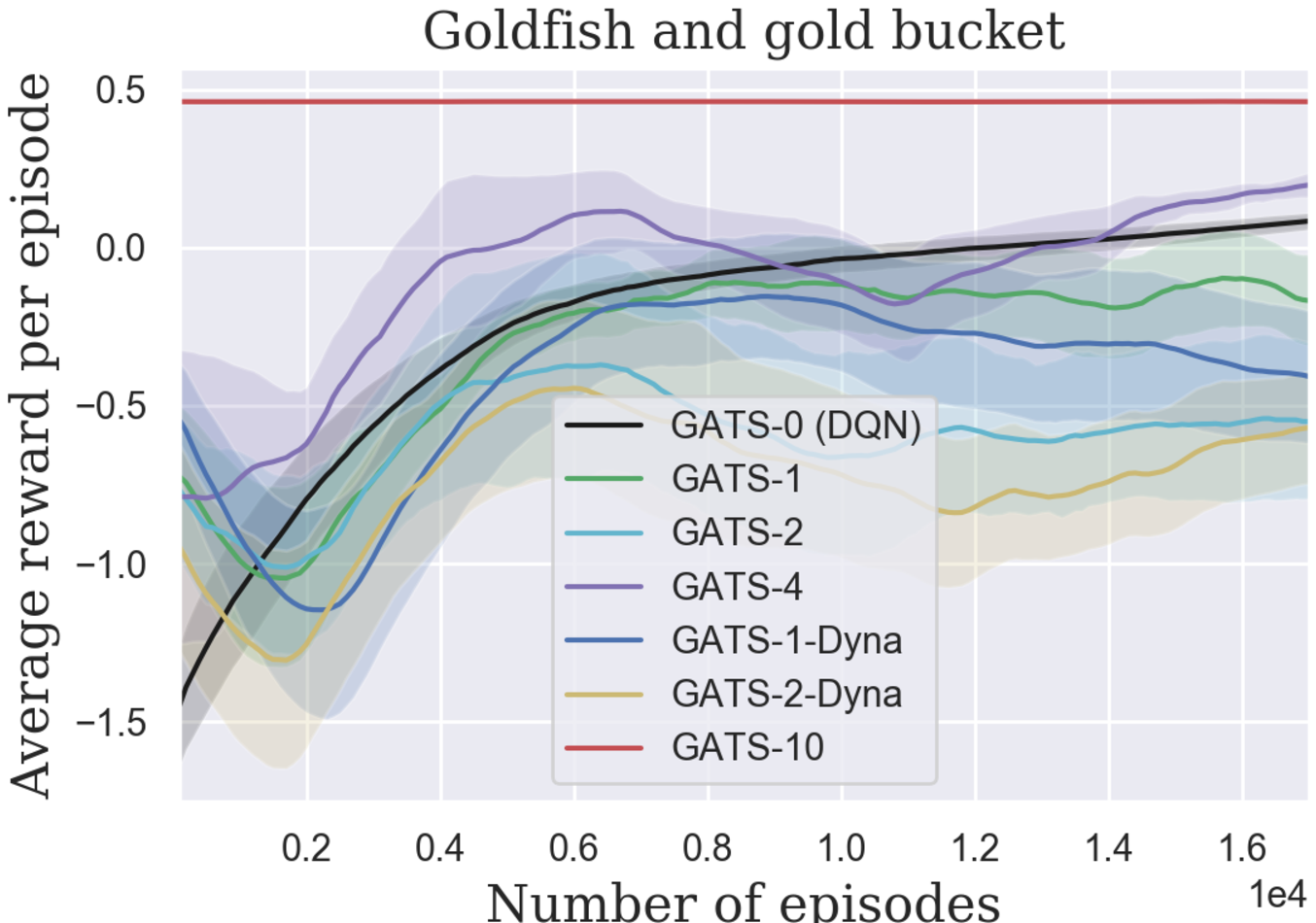}\\
    \hspace*{1.cm}
    (c)
    \end{minipage}    
    \caption{The Goldfish looks for a Gold bucket. The Q function is initialized such that 
    the yellow arrows represent the greedy action of each state. The red arrows are the actions suggested by MCTS depth two. \GATS is given the true model(a) \GATS locally prevents the goldfish from hitting the sharks 
    but also prevents learning from such events, therefore slows down the learning.(b) Even if the goldfish uses the prediction of the future event for further learning, as in Dyna-Q,the slow down issue still persists.(c) For a grid world of $10\times10$, and randomly initialized Q function,
    \GATS with depth of $10$ (\GATS-10) results in the highest return. Moreover, \GATS with nonzero depth locally saves the agent from hitting the sharks, but in the long run it degrades the performance.}
    \label{fig:goldfish}
    \vspace*{-0.5cm}
\end{figure}

\paragraph{Hypothesis on negative results:} 
In the following we reason why \GATS with short depth might not boost the performance even with perfect modeling with \GDM and \RP, despite reducing local bias. Consider a toy example described in Fig.~\ref{fig:goldfish}(a) where a fish starts with an initialization of the Q function such that the greedy action is represented in yellow arrows. We give the true model to \GATS (no modeling error). If the fish follows \DQN, it reaches the sharks quickly and receives a negative reward, update the Q function and learns the down action is not a good action. Now consider the depth 2 \GATS with the same Q-function initialization. When the agent reaches the step above the sharks, the MCTS roll-out informs the agent that there is a negative reward of going down and the agent chooses action right, following the \GATS action (red arrows). The agent keeps following the red arrows until it dies following $\varepsilon$-greedy. \GATS locally avoid bad states, but does this information globally get propagated? Consider that this negative reward happens much later in the trajectory
than for a \DQN agent. Therefore, many more updates are required for the agent to learn that action down is not a good action while this negative signal may even vanish due to the long update lengths. This shows how \GATS roll-outs can locally help to avoid (or reach) catastrophic (or good) events, but slow down global understanding.

As we suggested in our negative results, this problem can be solved by also putting the generated experiences in the replay buffer following Dyna-Q. Fig.~\ref{fig:goldfish}(b) illustrates the situation where the \GATS action in the third row before the sharks is ``right''. Therefore, similar to the previous setting, the agent keeps avoiding the sharks, choosing action right over and over and do not experience the shark negative signal. In this case, two-step roll-outs 
do not see the sharks, and thus Dyna-Q does not speed up the learning. In practice, especially with limited memory of the replay buffer and limited capacity of the function classes, it can be also difficult to tune Dyna-Q to work well.


To empirically test this hypothesis, we implemented the 10x10 version of \textit{Goldfish and gold bucket} environment where the \GATS agent has access to the true model of environment. We tested \GATS with depths of 0 (i.e., the plain \DQN), 1, 2, 4 and \GATS$+$Dyna-Q with the depth of 1 and 2 (even for this simple environment \GATS is computationally unpleasantly massive (two weeks of cpu machine for $\GATS-4$)). We use a randomly initialized Q network to learn the Q function. Figure~\ref{fig:goldfish}(c) represents the per episode return of different algorithms. Each episode has a maximum length of $100$ steps unless the agent either reaches the gold bucket (the reward of $+1$) or hit any of the sharks (the reward of $-1$). Here the discount factor is $0.99$ and cost of living is  $0.05$.

As expected, \GATS with the depth of 10 (the dimension of the grid) receives the highest return. We observe that \GATS with nonzero depth locally saves the agent and initially result in higher returns than DQN. However, in the long run, \GATS with short roll-outs (e.g. \GATS-1 and \GATS-2) degrade the performance, as seen in the later parts of the runs. Furthermore, we observe that the Dyna-Q approach also fails in improving performance. We train $\GATS{+}$Dyna-Q with both executed experiences and predicted ones. We observe that $\GATS{+}$Dyna-Q does not provide much benefit over \GATS. Consider that generated samples in the tree and real samples in $\GATS{+}$Dyna-Q are similar, resulting in repeated experiences in the replay buffer.  

\textbf{Concolusion:} For many complex applications, like hard Atari games, \GATS may require significantly longer roll-out depths with sophisticated Dyna-Q in order to perform well, which is computationally in-feasible for this study. However, the insights in designing near-perfect modeling algorithms and the extensive study of \GATS, highlight several key considerations and provide an important framework to effectively design algorithms for combining model-based and model-free reinforcement learning.

\section*{Acknowledgments}
K. Azizzadenesheli is supported in part by NSF Career Award CCF-1254106 and AFOSR YIP FA9550-15-1-0221. This research has been conducted when the first author was a visiting researcher at Stanford University and Caltech. A. Anandkumar is supported in part by Bren endowed chair, Darpa PAI, and Microsoft, Google, Adobe faculty fellowships, NSF Career Award CCF-1254106, and AFOSR YIP FA9550-15-1-0221. All the experimental study have been done using Caltech AWS credits grant.

\newpage
\bibliographystyle{apalike}
\bibliography{Master_arxiv}

\begin{thebibliography}{}

\bibitem[Abbasi{-}Yadkori et~al., 2011]{abbasi-yadkori2011improved}
Abbasi{-}Yadkori, Y., P{\'{a}}l, D., and Szepesv{\'{a}}ri, C. (2011).
\newblock Improved algorithms for linear stochastic bandits.
\newblock In {\em Advances in Neural Information Processing Systems 24 - NIPS},
  pages 2312--2320.

\bibitem[Abel et~al., 2016]{abel2016exploratory}
Abel, D., Agarwal, A., Diaz, F., Krishnamurthy, A., and Schapire, R.~E. (2016).
\newblock Exploratory gradient boosting for reinforcement learning in complex
  domains.
\newblock {\em arXiv}.

\bibitem[Antos et~al., 2008]{antos2008learning}
Antos, A., Szepesv{\'a}ri, C., and Munos, R. (2008).
\newblock Learning near-optimal policies with bellman-residual minimization
  based fitted policy iteration and a single sample path.
\newblock {\em Machine Learning}.

\bibitem[Arjovsky et~al., 2017]{arjovsky2017wasserstein}
Arjovsky, M., Chintala, S., and Bottou, L. (2017).
\newblock Wasserstein gan.
\newblock {\em arXiv preprint arXiv:1701.07875}.

\bibitem[Asmuth et~al., 2009]{asmuth2009bayesian}
Asmuth, J., Li, L., Littman, M.~L., Nouri, A., and Wingate, D. (2009).
\newblock A bayesian sampling approach to exploration in reinforcement
  learning.
\newblock In {\em Proceedings of the Twenty-Fifth Conference on Uncertainty in
  Artificial Intelligence}.

\bibitem[Auer, 2003]{auer2003using}
Auer, P. (2003).
\newblock Using confidence bounds for exploitation-exploration trade-offs.
\newblock {\em The Journal of Machine Learning Research}, 3:397--422.

\bibitem[Azizzadenesheli et~al., 2018]{azizzadenesheli2018efficient}
Azizzadenesheli, K., and Anandkumar, A. (2018).
\newblock Efficient exploration through bayesian deep q-networks.
\newblock {\em arXiv preprint arXiv:1802.04412}.

\bibitem[Azizzadenesheli et~al., 2016a]{azizzadenesheli2016rich}
Azizzadenesheli, K., Lazaric, A., and Anandkumar, A. (2016a).
\newblock Reinforcement learning in rich-observation mdps using spectral
  methods.
\newblock {\em arXiv preprint arXiv:1611.03907}.

\bibitem[Azizzadenesheli et~al., 2016b]{azizzadenesheli2016reinforcement}
Azizzadenesheli, K., Lazaric, A., and Anandkumar, A. (2016b).
\newblock Reinforcement learning of pomdps using spectral methods.
\newblock In {\em Proceedings of the 29th Annual Conference on Learning Theory
  (COLT)}.

\bibitem[Bartlett and Tewari, 2009]{bartlett2009regal}
Bartlett, P.~L. and Tewari, A. (2009).
\newblock {REGAL}: A regularization based algorithm for reinforcement learning
  in weakly communicating {MDP}s.
\newblock In {\em Proceedings of the 25th Annual Conference on Uncertainty in
  Artificial Intelligence}.

\bibitem[Bart{\'o}k et~al., 2014]{bartok2014partial}
Bart{\'o}k, G., Foster, D.~P., P{\'a}l, D., Rakhlin, A., and Szepesv{\'a}ri, C.
  (2014).
\newblock Partial monitoring—classification, regret bounds, and algorithms.
\newblock {\em Mathematics of Operations Research}.

\bibitem[Bellemare et~al., 2016]{bellemare2016unifying}
Bellemare, M., Srinivasan, S., Ostrovski, G., Schaul, T., Saxton, D., and
  Munos, R. (2016).
\newblock Unifying count-based exploration and intrinsic motivation.
\newblock In {\em Advances in Neural Information Processing Systems}, pages
  1471--1479.

\bibitem[Bellemare et~al., 2013]{bellemare2013arcade}
Bellemare, M.~G., Naddaf, Y., Veness, J., and Bowling, M. (2013).
\newblock The arcade learning environment: An evaluation platform for general
  agents.
\newblock {\em J. Artif. Intell. Res.(JAIR)}.

\bibitem[Brafman and Tennenholtz, 2003]{brafman2003r}
Brafman, R.~I. and Tennenholtz, M. (2003).
\newblock R-max-a general polynomial time algorithm for near-optimal
  reinforcement learning.
\newblock {\em The Journal of Machine Learning Research}, 3:213--231.

\bibitem[Brockman et~al., 2016]{1606.01540}
Brockman, G., Cheung, V., Pettersson, L., Schneider, J., Schulman, J., Tang,
  J., and Zaremba, W. (2016).
\newblock Openai gym.

\bibitem[David~Ha, 2018]{DavidHa}
David~Ha, J.~S. (2018).
\newblock World models.
\newblock {\em arXiv preprint arXiv:1803.10122}.

\bibitem[Fortunato et~al., 2017]{fortunato2017noisy}
Fortunato, M., Azar, M.~G., Piot, B., Menick, J., Osband, I., Graves, A., Mnih,
  V., Munos, R., Hassabis, D., Pietquin, O., et~al. (2017).
\newblock Noisy networks for exploration.
\newblock {\em arXiv preprint arXiv:1706.10295}.

\bibitem[Goodfellow et~al., 2014]{goodfellow2014generative}
Goodfellow, I., Pouget-Abadie, J., Mirza, M., Xu, B., Warde-Farley, D., Ozair,
  S., Courville, A., and Bengio, Y. (2014).
\newblock Generative adversarial nets.
\newblock In {\em Advances in neural information processing systems}, pages
  2672--2680.

\bibitem[Gulrajani et~al., 2017]{gulrajani2017improved}
Gulrajani, I., Ahmed, F., Arjovsky, M., Dumoulin, V., and Courville, A.~C.
  (2017).
\newblock Improved training of wasserstein gans.
\newblock In {\em Advances in Neural Information Processing Systems}, pages
  5769--5779.

\bibitem[Guo et~al., 2014]{guo2014deep}
Guo, X., Singh, S., Lee, H., Lewis, R.~L., and Wang, X. (2014).
\newblock Deep learning for real-time atari game play using offline monte-carlo
  tree search planning.
\newblock In {\em Advances in neural information processing systems}, pages
  3338--3346.

\bibitem[Holland et~al., 2018]{holland2018effect}
Holland, G.~Z., Talvitie, E.~J., and Bowling, M. (2018).
\newblock The effect of planning shape on dyna-style planning in
  high-dimensional state spaces.
\newblock {\em arXiv preprint arXiv:1806.01825}.

\bibitem[Isola et~al., 2017]{isola2017image}
Isola, P., Zhu, J.-Y., Zhou, T., and Efros, A.~A. (2017).
\newblock Image-to-image translation with conditional adversarial networks.
\newblock {\em arXiv preprint}.

\bibitem[Jaksch et~al., 2010]{jaksch2010near}
Jaksch, T., Ortner, R., and Auer, P. (2010).
\newblock Near-optimal regret bounds for reinforcement learning.
\newblock {\em Journal of Machine Learning Research}.

\bibitem[Kakade, 2002]{kakade2002natural}
Kakade, S.~M. (2002).
\newblock A natural policy gradient.
\newblock In {\em Advances in neural information processing systems}.

\bibitem[Kearns et~al., 2002]{kearns2002sparse}
Kearns, M., Mansour, Y., and Ng, A.~Y. (2002).
\newblock A sparse sampling algorithm for near-optimal planning in large markov
  decision processes.
\newblock {\em Machine Learning}, 49(2-3):193--208.

\bibitem[Kearns and Singh, 2002]{kearns2002near}
Kearns, M. and Singh, S. (2002).
\newblock Near-optimal reinforcement learning in polynomial time.
\newblock {\em Machine Learning}, 49(2-3):209--232.

\bibitem[Kocsis and Szepesv{\'a}ri, 2006]{kocsis2006bandit}
Kocsis, L. and Szepesv{\'a}ri, C. (2006).
\newblock Bandit based monte-carlo planning.
\newblock In {\em Machine Learning: ECML 2006}, pages 282--293. Springer.

\bibitem[Lagoudakis and Parr, 2003]{lagoudakis2003least}
Lagoudakis, M.~G. and Parr, R. (2003).
\newblock Least-squares policy iteration.
\newblock {\em Journal of machine learning research}, 4(Dec):1107--1149.

\bibitem[Levine~et al., 2016]{levine2016end}
Levine~et al., S. (2016).
\newblock End-to-end training of deep visuomotor policies.
\newblock {\em JMLR}.

\bibitem[Lipton et~al., 2016]{lipton2016combating}
Lipton, Z.~C., Azizzadenesheli, K., Gao, J., Li, L., Chen, J., and Deng, L.
  (2016).
\newblock Combating reinforcement learning's sisyphean curse with intrinsic
  fear.
\newblock {\em arXiv preprint arXiv:1611.01211}.

\bibitem[Lipton et~al., 2018]{lipton2016efficient}
Lipton, Z.~C., Gao, J., Li, L., Li, X., Ahmed, F., and Deng, L. (2018).
\newblock Efficient exploration for dialogue policy learning with bbq networks
  \& replay buffer spiking.
\newblock {\em AAAI}.

\bibitem[Machado et~al., 2017]{machado2017revisiting}
Machado, M.~C., Bellemare, M.~G., Talvitie, E., Veness, J., Hausknecht, M., and
  Bowling, M. (2017).
\newblock Revisiting the arcade learning environment: Evaluation protocols and
  open problems for general agents.
\newblock {\em arXiv preprint arXiv:1709.06009}.

\bibitem[Mathieu et~al., 2015]{mathieu2015deep}
Mathieu, M., Couprie, C., and LeCun, Y. (2015).
\newblock Deep multi-scale video prediction beyond mean square error.
\newblock {\em arXiv preprint arXiv:1511.05440}.

\bibitem[Mirza and Osindero, 2014]{mirza2014conditional}
Mirza, M. and Osindero, S. (2014).
\newblock Conditional generative adversarial nets.
\newblock {\em arXiv preprint arXiv:1411.1784}.

\bibitem[Miyato et~al., 2018]{miyato2018spectral}
Miyato, T., Kataoka, T., Koyama, M., and Yoshida, Y. (2018).
\newblock Spectral normalization for generative adversarial networks.
\newblock {\em arXiv preprint arXiv:1802.05957}.

\bibitem[Mnih et~al., 2015]{mnih2015human}
Mnih, V., Kavukcuoglu, K., Silver, D., Rusu, A.~A., Veness, J., Bellemare,
  M.~G., Graves, A., Riedmiller, M., Fidjeland, A.~K., Ostrovski, G., et~al.
  (2015).
\newblock Human-level control through deep reinforcement learning.
\newblock {\em Nature}.

\bibitem[Odena et~al., 2016]{odena2016conditional}
Odena, A., Olah, C., and Shlens, J. (2016).
\newblock Conditional image synthesis with auxiliary classifier gans.
\newblock {\em arXiv preprint arXiv:1610.09585}.

\bibitem[Oh et~al., 2015]{oh2015action}
Oh, J., Guo, X., Lee, H., Lewis, R.~L., and Singh, S. (2015).
\newblock Action-conditional video prediction using deep networks in atari
  games.
\newblock In {\em Advances in Neural Information Processing Systems}.

\bibitem[Oh et~al., 2017]{oh2017value}
Oh, J., Singh, S., and Lee, H. (2017).
\newblock Value prediction network.
\newblock In {\em Advances in Neural Information Processing Systems}, pages
  6120--6130.

\bibitem[Osband et~al., 2016]{osband2016deep}
Osband, I., Blundell, C., Pritzel, A., and Van~Roy, B. (2016).
\newblock Deep exploration via bootstrapped dqn.
\newblock In {\em Advances in Neural Information Processing Systems}.

\bibitem[Ostrovski et~al., 2017]{ostrovski2017count}
Ostrovski, G., Bellemare, M.~G., Oord, A. v.~d., and Munos, R. (2017).
\newblock Count-based exploration with neural density models.
\newblock {\em arXiv preprint arXiv:1703.01310}.

\bibitem[Schulman et~al., 2015]{schulman2015trust}
Schulman, J., Levine, S., Abbeel, P., Jordan, M., and Moritz, P. (2015).
\newblock Trust region policy optimization.
\newblock In {\em Proceedings of the 32nd International Conference on Machine
  Learning (ICML-15)}.

\bibitem[Schweitzer and Seidmann, 1985]{schweitzer1985generalized}
Schweitzer, P.~J. and Seidmann, A. (1985).
\newblock Generalized polynomial approximations in markovian decision
  processes.
\newblock {\em Journal of mathematical analysis and applications},
  110(2):568--582.

\bibitem[{Silver et al.}, 2016]{silver2016mastering}
{Silver et al.}, D. (2016).
\newblock Mastering the game of go with deep neural networks and tree search.
\newblock {\em Nature}.

\bibitem[Sutton, 1990]{sutton1990integrated}
Sutton, R.~S. (1990).
\newblock Integrated architectures for learning, planning, and reacting based
  on approximating dynamic programming.
\newblock In {\em Machine Learning Proceedings 1990}, pages 216--224. Elsevier.

\bibitem[Thrun and Schwartz, 1993]{thrun1993issues}
Thrun, S. and Schwartz, A. (1993).
\newblock Issues in using function approximation for reinforcement learning.
\newblock In {\em Proceedings of the 1993 Connectionist Models Summer School
  Hillsdale, NJ. Lawrence Erlbaum}.

\bibitem[Van~Hasselt et~al., 2016]{van2016deep}
Van~Hasselt, H., Guez, A., and Silver, D. (2016).
\newblock Deep reinforcement learning with double q-learning.
\newblock In {\em AAAI}.

\bibitem[Wahlstr{\"o}m et~al., 2015]{wahlstrom2015pixels}
Wahlstr{\"o}m, N., Sch{\"o}n, T.~B., and Deisenroth, M.~P. (2015).
\newblock From pixels to torques: Policy learning with deep dynamical models.
\newblock {\em arXiv preprint arXiv:1502.02251}.

\bibitem[Weber et~al., 2017]{weber2017imagination}
Weber, T., Racani{\`e}re, S., Reichert, D.~P., Buesing, L., Guez, A., Rezende,
  D.~J., Badia, A.~P., Vinyals, O., Heess, N., Li, Y., et~al. (2017).
\newblock Imagination-augmented agents for deep reinforcement learning.
\newblock {\em arXiv}.

\end{thebibliography}

\newpage
\appendix
\section*{Appendix}


\begin{figure*}[ht]
\centering
\hspace*{-1.2cm}
\begin{tabular}{cc}
\vspace*{-.3cm}
&\includegraphics[width=14.7cm,height=3.9cm]{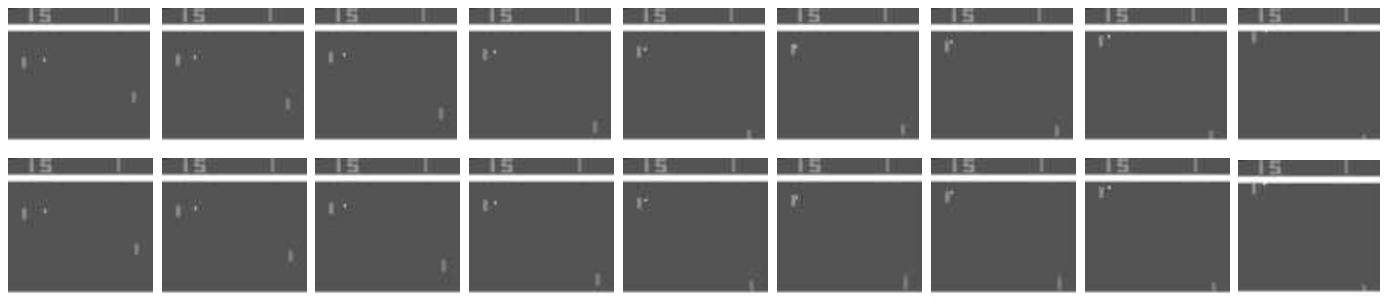}\\
\vspace*{-.7cm}
&\includegraphics[width=15.5cm,height=3.5cm]{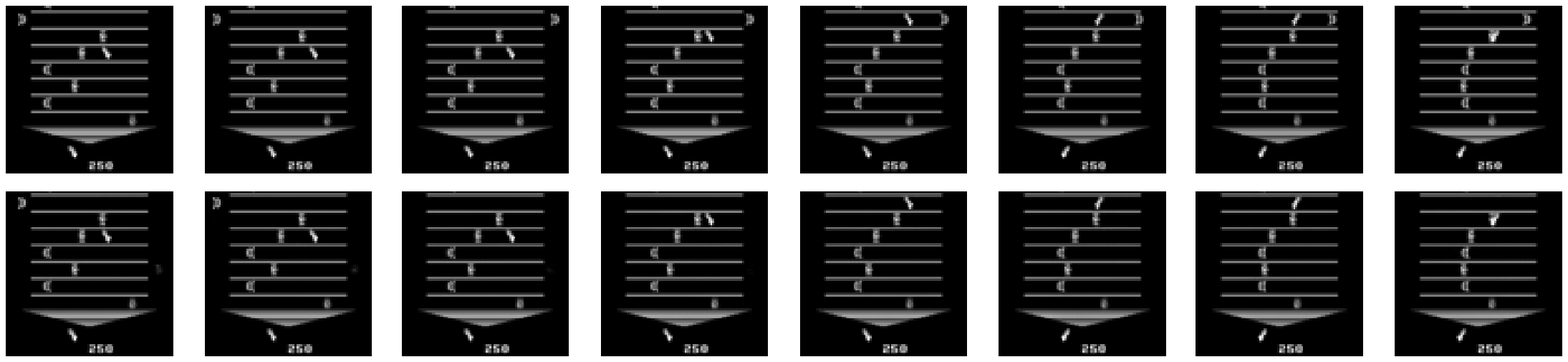}\\
\vspace*{-.6cm}
&\includegraphics[width=15.5cm,height=3.5cm]{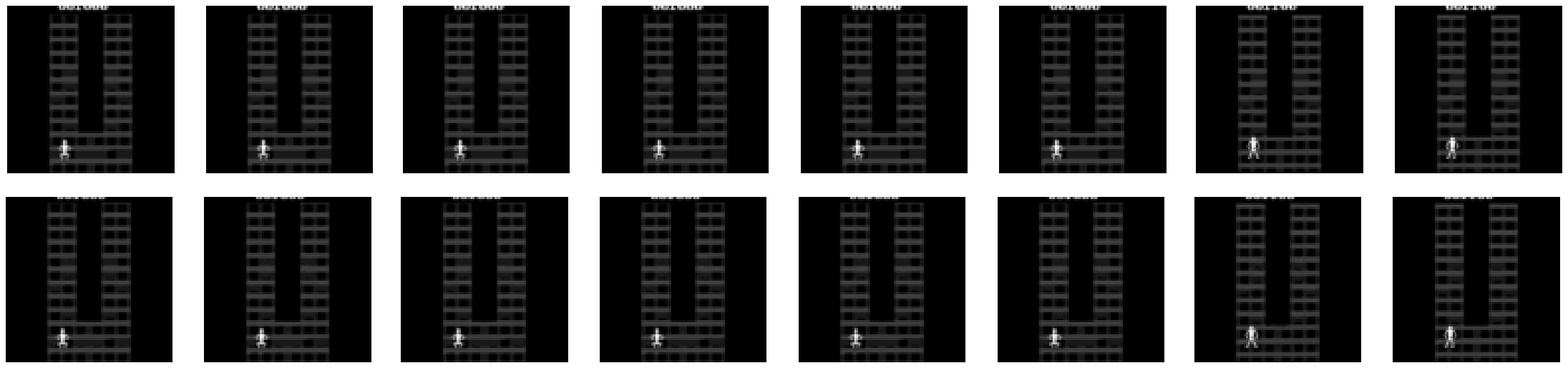}\\
\vspace*{-.2cm}
&\hspace{-.2cm}\includegraphics[width=15.2cm,height=3.5cm]{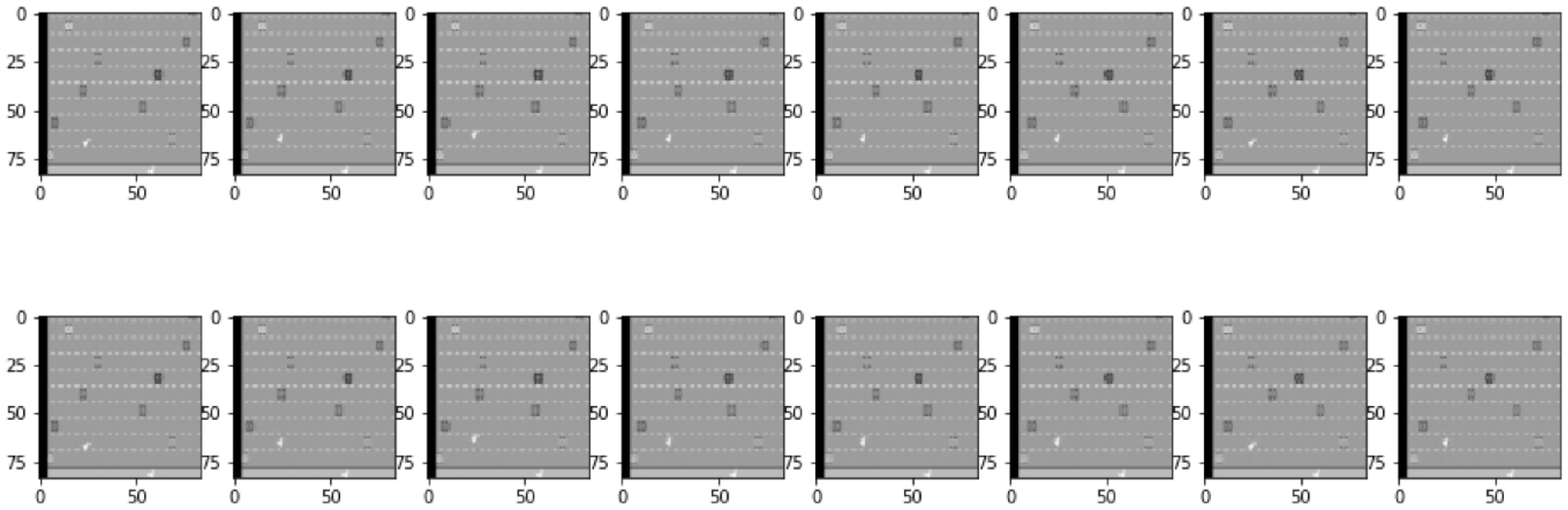}\\
&\hspace{-.2cm}\includegraphics[width=15.2cm,height=3.5cm]{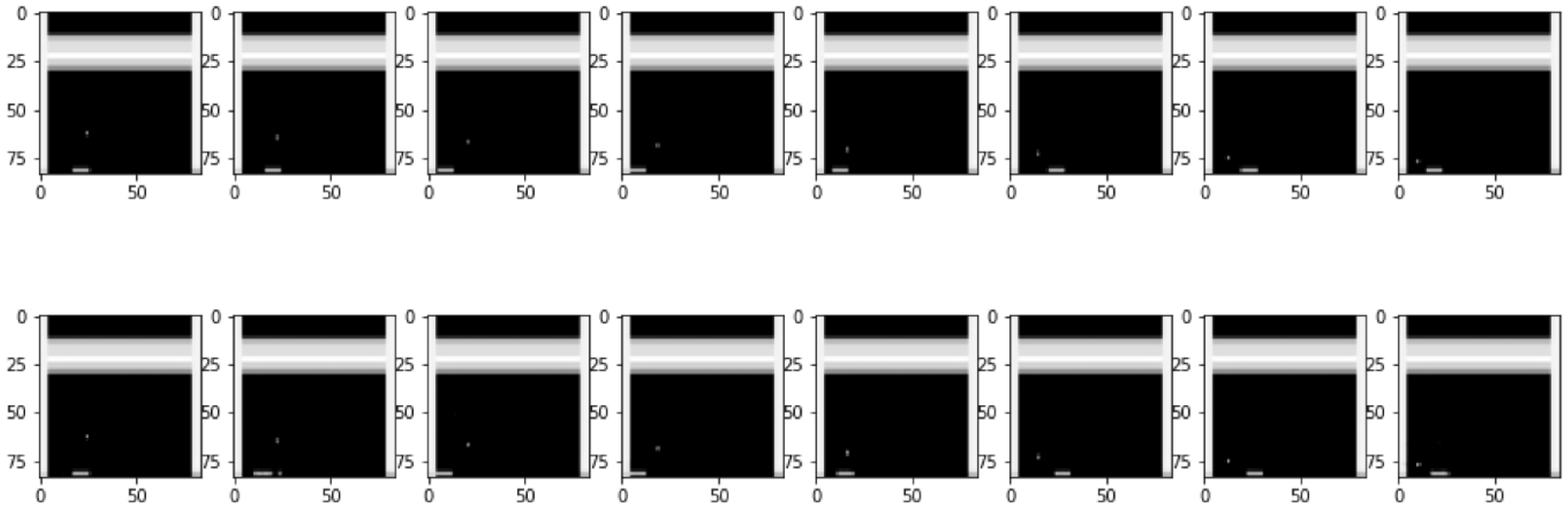}
\end{tabular}
\vspace*{-0.5cm}
\caption{On the performance of the proposed \GDM. Given four consecutive frames of Atari games, and a sequence of eight actions, \GDM generates  sequences of the future frames almost identical to the real frames. First row: A sequence of real frames. Second row: a corresponding sequence of generated frames}
  \label{fig:exp}
\vspace*{-0.5cm}
\end{figure*}

\section{Proof of Proposition  \ref{thm:BV} }\label{sec:proof}
Let's restate the estimated returns with the learned model as follows;

\begin{align*}
\E_{\piroll,\GDM,\RP}\left[\sum_{h=0}^{H-1}\gamma^h\wh r_h+\gamma^H\max_a\wh{Q}({x}_H,a)\Big| x\right]:=\!\!\!\!\!\!&
\sum_{x_i,a_i,\forall i\in[1,.,H]}\!\!\!\!\!\!\wh T(x_1|x,a_1)\piroll(a_1|x)\prod_{j=2}^{H}\wh T(x_j|x_{j-1},a_j)\piroll(a_j|x_{j-1})\\
&\quad\left(\wh r(x,a_1)+\sum_{j=2}^{H}\gamma^{j-1}\wh r(x_{j-1},a_j)+\gamma^H\max_a\wh{Q}({x}_H,a)\right)
\end{align*}
Now consider the following lemma;
\begin{lemma}[Deviation in Q-function]\label{lem:Q}
Define $e_Q$ as the uniform bound on error in the estimation of the Q function, such that $|Q(x,a)- \wh{Q}(x,a)|\leq e_Q~,~\forall x,a$. Then;
\begin{align*}
 \Big|\max_a\wh{Q}({x},a)- \max_a{Q}({x},a)\Big|\leq e_Q
\end{align*}
\end{lemma}

\begin{proof}
For a given state $x$, define $\wt a_1(x):=\arg\max_a\wh Q(x,a)$ and $\wt a_2(x):=\arg\max_a Q(x,a)$. Then;
\begin{align*}
\wh Q(x,\wt a_1(x))- Q(x,\wt a_2(x))&=\wh Q(x,\wt a_1(x))- Q(x,\wt a_1(x))+ Q(x,\wt a_1(x)) -Q(x,\wt a_2(x))\nonumber\\
&\leq \wh Q(x,\wt a_1(x)) -Q(x,\wt a_1(x))\leq e_Q
\end{align*}
since $Q(x,\wt a_1(x)) -Q(x,\wt a_2(x))\leq 0$. 

With similar argument, we have;
\begin{align*}
 \wh Q(x,\wt a_1(x))- Q(x,\wt a_2(x))&= \wh Q(x,\wt a_1(x))- \wh Q(x,\wt a_2(x))+ \wh Q(x_H,\wt a_2(x)) -Q(x,\wt a_2(x))\nonumber\\
&\geq \wh Q(x,\wt a_2(x))-Q(x,\wt a_2(x))\geq -e_Q
\end{align*}
since $\wh Q(x,\wt a_1(x))- \wh Q(x,\wt a_2(x))\geq 0$. Therefore;
\begin{align*}
 -e_Q\leq\wh Q(x,\wt a_1(x))- Q(x,\wt a_2(x))\leq e_Q
\end{align*}
resulting in Lemma~\ref{lem:Q}.
\end{proof}

In the remaining proof, we repeatedly apply the addition and subtraction technique to upper bound the error. To show how we apply this technique, we illustrate how we derive the error terms $T(x_1|x,a_1) - \wh T(x_1|x, a_1)$ and $r(x,a_1)-\wh r(x,a_1)$ for the first time step in detail. 


Let us restate the objective that we desire to upper bound:

\begin{align}\label{eq:error_main}
&\!\!\!\!\!\!\!\!\Big|\E_{\piroll}\left[\sum_{h=0}^{H-1}\gamma^hr_h+\gamma^H\max_a{Q}({x}_H,a)\Big| x\right]-\E_{\piroll,\GDM,\RP}\left[\sum_{h=0}^{H}\gamma^h\wh r_h+\gamma^H\max_a\wh{Q}({x}_H,a)\Big| x\right]\Big|\nonumber\\
&\!\!\!\!\!\!\!\!=\Big|\!\!\!\!\!\!
\sum_{x_i,a_i,\forall i\in[1,.,H]}\!\!\!\!\!\!\!\!\!\!\!\! T(x_1|x,a_1)\piroll(a_1|x)\prod_{j=2}^{H} T(x_j|x_{j-1},a_j)\piroll(a_j|x_{j-1})
\left( r(x,a_1)+\sum_{j=2}^{H}\gamma^{j-1} r(x_{j-1},a_j)+\gamma^H\max_a{Q}({x}_H,a)\right)\nonumber\\
&\!\!\!\!\!\!\!\!-\!\!\!\!\!\!
\sum_{x_i,a_i,\forall i\in[1,.,H]}\!\!\!\!\!\!\!\!\!\!\!\!\wh T(x_1|x,a_1)\piroll(a_1|x)\prod_{j=2}^{H}\wh T(x_j|x_{j-1},a_j)\piroll(a_j|x_{j-1})
\left(\wh r(x,a_1)+\sum_{j=2}^{H}\gamma^{j-1}\wh r(x_{j-1},a_j)+\gamma^H\max_a\wh{Q}({x}_H,a)\right)\Big|
\end{align}

Then, we add and subtract the following term. 
\begin{align*}
\sum_{x_i,a_i,\forall i\in[1,.,H]}\!\!\!\!\!\!\!\!\!\!\!\! \wh T(x_1|x,a_1)\piroll(a_1|x)\prod_{j=2}^{H} T(x_j|x_{j-1},a_j)\piroll(a_j|x_{j-1})
\left( r(x,a_1)+\sum_{j=2}^{H}\gamma^{j-1} r(x_{j-1},a_j)+\gamma^H\max_a{Q}({x}_H,a)\right)
\end{align*}
Notice this term differs from the first term of Eq.~\ref{eq:error_main} just in the transition kernel of the first time step, i.e., $T(x_1|x,a_1) \xrightarrow{}\wh T(x_1|x,a_1)$. Thus, we have;
\begin{align} \label{eq:ap1}
&\!\!\!\!\!\Big|\E_{\piroll}\left[\sum_{h=0}^{H-1}\gamma^hr_h+\gamma^H\max_a{Q}({x}_H,a)\Big| x\right]-\E_{\piroll,\GDM,\RP}\left[\sum_{h=0}^{H}\gamma^h\wh r_h+\gamma^H\max_a\wh{Q}({x}_H,a)\Big| x\right]\Big|\nonumber\\
&\!\!\!\!\!=\Big|\!\!\!\!\!\!
\sum_{x_i,a_i,\forall i\in[1,.,H]}\!\!\!\!\!\!\!\!\!\!\!\! T(x_1|x,a_1)\piroll(a_1|x)\prod_{j=2}^{H} T(x_j|x_{j-1},a_j)\piroll(a_j|x_{j-1})\left( r(x,a_1)+\sum_{j=2}^{H}\gamma^{j-1} r(x_{j-1},a_j)+\gamma^H\max_a{Q}({x}_H,a)\right)\nonumber\\
&\!\!\!\!\!-\!\!\!\!\!\!
\sum_{x_i,a_i,\forall i\in[1,.,H]}\!\!\!\!\!\!\!\!\!\!\!\! \wh T(x_1|x,a_1)\piroll(a_1|x)\prod_{j=2}^{H} T(x_j|x_{j-1},a_j)\piroll(a_j|x_{j-1})
\left( r(x,a_1)+\sum_{j=2}^{H}\gamma^{j-1} r(x_{j-1},a_j)+\gamma^H\max_a{Q}({x}_H,a)\right)\nonumber\\
&\!\!\!\!\!+\!\!\!\!\!\!
\sum_{x_i,a_i,\forall i\in[1,.,H]}\!\!\!\!\!\!\!\!\!\!\!\! \wh T(x_1|x,a_1)\piroll(a_1|x)\prod_{j=2}^{H} T(x_j|x_{j-1},a_j)\piroll(a_j|x_{j-1})
\left( r(x,a_1)+\sum_{j=2}^{H}\gamma^{j-1} r(x_{j-1},a_j)+\gamma^H\max_a{Q}({x}_H,a)\right)\nonumber\\
&\!\!\!\!\!-\!\!\!\!\!\!
\sum_{x_i,a_i,\forall i\in[1,.,H]}\!\!\!\!\!\!\!\!\!\!\!\!\wh T(x_1|x,a_1)\piroll(a_1|x)\prod_{j=2}^{H}\wh T(x_j|x_{j-1},a_j)\piroll(a_j|x_{j-1})\left(\wh r(x,a_1)+\sum_{j=2}^{H}\gamma^{j-1}\wh r(x_{j-1},a_j)+\gamma^H\max_a\wh{Q}({x}_H,a)\right)\Big|
\end{align}

We derive the error term $T(x_1|x,a_1) - \wh T(x_1|x, a_1)$ from the first two terms of Eq. ~\ref{eq:ap1}. Notice that all the parameters of the first two terms are the same except the transition kernel for the first state. We can thus refactor the first two terms of Eq.~\ref{eq:ap1} as:

\begin{align*}
\sum_{x_i,a_i,\forall i\in[1,.,H]}\!\!\!\!\!\!\!\!\!\!\!\! \left(T(x_1|x,a_1) - \wh T(x_1|x, a_1)\right)&\piroll(a_1|x)\prod_{j=2}^{H} T(x_j|x_{j-1},a_j)\piroll(a_j|x_{j-1})\\
&\left( r(x,a_1)+\sum_{j=2}^{H}\gamma^{j-1} r(x_{j-1},a_j)+\gamma^H\max_a{Q}({x}_H,a)\right)
\end{align*}

We then expand the third and fourth terms of Eq.~\ref{eq:ap1} to derive the error term $r(x, a_1) - \wh r(x, a_1)$ and a remainder. To do this, we add and subtract the following term which is the same as the third term in Eq.~\ref{eq:ap1} except it differs in the reward of the first time step:

\begin{align*}
    &\!\!\!\!\!
\sum_{x_i,a_i,\forall i\in[1,.,H]}\!\!\!\!\!\! \wh T(x_1|x,a_1)\piroll(a_1|x)\prod_{j=2}^{H} T(x_j|x_{j-1},a_j)\piroll(a_j|x_{j-1})\left( \wh r(x,a_1)+\sum_{j=2}^{H}\gamma^{j-1} r(x_{j-1},a_j)+\gamma^H\max_a{Q}({x}_H,a)\right)
\end{align*}


We can thus express the third and fourth terms of Eq.~\ref{eq:ap1} along with the addition and subtraction terms as:
\begin{align} \label{eq:ap2}
&\!\!\!\!\!
\sum_{x_i,a_i,\forall i\in[1,.,H]}\!\!\!\!\!\!\!\!\!\!\!\! \wh T(x_1|x,a_1)\piroll(a_1|x)\prod_{j=2}^{H} T(x_j|x_{j-1},a_j)\piroll(a_j|x_{j-1})
\left( r(x,a_1)+\sum_{j=2}^{H}\gamma^{j-1} r(x_{j-1},a_j)+\gamma^H\max_a{Q}({x}_H,a)\right)\nonumber\\
&\!\!\!\!\!
-\!\!\!\!\!\!\sum_{x_i,a_i,\forall i\in[1,.,H]}\!\!\!\!\!\!\!\!\!\!\!\! \wh T(x_1|x,a_1)\piroll(a_1|x)\prod_{j=2}^{H} T(x_j|x_{j-1},a_j)\piroll(a_j|x_{j-1})\left( \wh r(x,a_1)+\sum_{j=2}^{H}\gamma^{j-1} r(x_{j-1},a_j)+\gamma^H\max_a{Q}({x}_H,a)\right)\nonumber\\
&\!\!\!\!\!
+\!\!\!\!\!\!\sum_{x_i,a_i,\forall i\in[1,.,H]}\!\!\!\!\!\!\!\!\!\!\!\! \wh T(x_1|x,a_1)\piroll(a_1|x)\prod_{j=2}^{H} T(x_j|x_{j-1},a_j)\piroll(a_j|x_{j-1})
\left( \wh r(x,a_1)+\sum_{j=2}^{H}\gamma^{j-1} r(x_{j-1},a_j)+\gamma^H\max_a{Q}({x}_H,a)\right)\nonumber\\
&\!\!\!\!\!-\!\!\!\!\!\!
\sum_{x_i,a_i,\forall i\in[1,.,H]}\!\!\!\!\!\!\!\!\!\!\!\!\wh T(x_1|x,a_1)\piroll(a_1|x)\prod_{j=2}^{H}\wh T(x_j|x_{j-1},a_j)\piroll(a_j|x_{j-1})
\left(\wh r(x,a_1)+\sum_{j=2}^{H}\gamma^{j-1}\wh r(x_{j-1},a_j)+\gamma^H\max_a\wh{Q}({x}_H,a)\right)
\end{align}

Notice that the first two terms in Eq.~\ref{eq:ap2} are the same except in in the first reward term, from which we derive the error term $ r(x,a_1)-\wh r(x,a_1)$. We refactor the first two terms in Eq.~\ref{eq:ap2} as: 

\begin{align*}
&\!\!\!\!\!\!
\sum_{x_i,a_i,\forall i\in[1,.,H]}\!\!\!\!\!\!\!\!\!\!\!\! (r(x,a_1) - \wh r(x, a_1)) \wh T(x_1|x,a_1)\piroll(a_1|x)\prod_{j=2}^{H} T(x_j|x_{j-1},a_j)\piroll(a_j|x_{j-1})\\
&\!\!\!\!\!\!
\end{align*}

Finally, we have the remaining last two terms of Eq.~\ref{eq:ap2}.
\begin{align*}
&\!\!\!\!\!\sum_{x_i,a_i,\forall i\in[1,.,H]}\!\!\!\!\!\!\!\!\!\!\!\! \wh T(x_1|x,a_1)\piroll(a_1|x)\prod_{j=2}^{H} T(x_j|x_{j-1},a_j)\piroll(a_j|x_{j-1})
\left( \wh r(x,a_1)+\sum_{j=2}^{H}\gamma^{j-1} r(x_{j-1},a_j)+\gamma^H\max_a{Q}({x}_H,a)\right)\\
&\!\!\!\!\!-\!\!\!\!\!\!
\sum_{x_i,a_i,\forall i\in[1,.,H]}\!\!\!\!\!\!\!\!\!\!\!\!\wh T(x_1|x,a_1)\piroll(a_1|x)\prod_{j=2}^{H}\wh T(x_j|x_{j-1},a_j)\piroll(a_j|x_{j-1})
\left(\wh r(x,a_1)+\sum_{j=2}^{H}\gamma^{j-1}\wh r(x_{j-1},a_j)+\gamma^H\max_a\wh{Q}({x}_H,a)\right)
\end{align*}

We repeatedly expand this remainder for the following time steps using the same steps as described above to derive the full bound. Following this procedure, we have:

\begin{align*}
&\Big|\E_{\piroll}\left[\sum_{h=0}^{H-1}\gamma^hr_h+\gamma^H\max_a{Q}({x}_H,a)\Big| x\right]-\E_{\piroll,\GDM,\RP}\left[\sum_{h=0}^{H}\gamma^h\wh r_h+\gamma^H\max_a\wh{Q}({x}_H,a)\Big| x\right]\Big|\\
&\leq\sum_{x_i,a_i,\forall i\in[H]}\left|T(x_1|x,a_1)-\wh{T}({x}_1|x,a_1)\right|\piroll(a_1|x)\\
&\quad\quad\quad\quad\quad\left(r(x,a_1)+\sum_{j=2}^{H}\gamma^{j-1}r(x_{j-1},a_j)+\gamma^H\max_a{Q}({x}_H,a)\right)\prod_{j=2}^{H}T(x_j|x_{j-1},a_j)\piroll(a_j|x_{j-1})\nonumber\\
&+\sum_{x_i,a_i,\forall i\in[H]}\wh{T}({x}_1|x,a_1)\piroll(a_1|x)\left(|r(x,a_1)-\wh r(x,a_1)|\right)\prod_{j=2}^{H}T(x_j|x_{j-1},a_j)\piroll(a_j|x_{j-1})\nonumber\\
&\quad\quad+\sum_{j=2}^{H}\sum_{x_h,a_h,\forall i\in[H]}\!\!\!\wh{T}({x}_1|x,a_1)\piroll(a_1|x)\Big|T(x_j|x_{j-1},a_j)-\!\wh{T}({x}_j|x_{j-1},a_j)\Big|\nonumber\\
&\quad\quad\quad\quad\quad\quad\quad\quad\left(\sum_{h=j+1}^{H}\gamma^{h-1} r(x_{h-1},a_h)+\gamma^H\max_a{Q}({x}_H,a)\right)\\
&\quad\quad\quad\quad\quad\quad\quad\quad\prod_{h=2}^{j-1}\wh T({x}_h|x_{h-1},a_h)\piroll(a_h|{x}_{h-1})\prod_{h=j+1}^{H}T(x_h|x_{h-1},a_h)\piroll(a_h|x_{h-1})\\
&\quad\quad+\sum_{j=2}^{H}\gamma^{j-1}\sum_{x_h,a_h,\forall i\in[H]}\wh{T}({x}_1|x,a_1)\piroll(a_1|x)\left(\Big|r(x_{j-1},a_{j})-\wh r(x_{j-1},a_{j})\Big|\right)\\
&\quad\quad\quad\quad\quad\quad\quad\quad\prod_{h=2}^{j}\wh T({x}_h|x_{h-1},a_h)\piroll(a_h|{x}_{h-1})\prod_{h=j+1}^{H}T(x_h|x_{h-1},a_h)\piroll(a_h|x_{h-1})\\
&+\sum_{x_h,a_h,\forall i\in[H]}\wh{T}({x}_1|x,a_1)\piroll(a_1|x)\prod_{h=2}^{H}\wh T({x}_h|x_{h-1},a_h)\piroll(a_h|x_{h-1})\gamma^{H}\Big|\max_a{Q}({x}_H,a)-\max_a\wh{Q}({x}_H,a))\Big|\\
\end{align*}

As a result,
\begin{align*}
\Big|\E_{\piroll}\left[\sum_{h=0}^{H-1}\gamma^hr_h+\gamma^H\max_a{Q}({x}_H,a)\Big| x\right]&-\E_{\piroll,\GDM,\RP}\left[\sum_{h=0}^{H-1}\gamma^h\wh r_h+\gamma^H\max_a\wh{Q}({x}_H,a)\Big| x\right]\Big|\\
&\leq
\!\sum_{i=1}^{H}\!\! \gamma^{i-1}\frac{1-\gamma^{H+1-i}}{1-\gamma}  e_T+ \!\!\sum_{i=1}^{H}\!\!\gamma^{i-1} e_R\!+\gamma^He_Q\\
&\leq\!\frac{1-\gamma^{H}+H\gamma^{H}(1-\gamma)}{(1-\gamma)^2}e_T+\frac{1-\gamma^{H}}{1-\gamma}e_R+\gamma^He_Q
\vspace*{-0.4cm}
\end{align*}
Therefore, the proposition follows.\\

\vspace*{-0.6cm}

\section{Bias-Variance in Q function}\label{apx:BV}
To observe the existing bias and variance in $Q_{\theta}$, we run solely \DQN on the game Pong, for $20M$ frame steps. Fig.~\ref{fig:pongQ} shows 4 consecutive frames where the agent receives a negative score and Table.~\ref{fig:pongQ} shows the estimated Q values by \DQN for these steps. As we observe in  Fig.~\ref{fig:pongQ} and Table.~\ref{fig:pongQ}, at the time step $t$, the estimated Q value of all the actions are almost the same. The agent takes the \textit{down} action and the environment goes to the next state $t+1$. The second row of Table.~\ref{fig:pongQ} expresses the Q value of the actions at this new state. Since this transition does not carry any reward and the discount factor is close to 1, ($\gamma=0.99$) we expect the max Q values at time step $t+1$ to be close the Q values of action \textit{down}, but it is very different.
\begin{figure}[ht]
\hspace{0.5cm}
\vspace*{-0.0cm}
   \begin{minipage}{0.1\textwidth}
\includegraphics[width=1\textwidth]{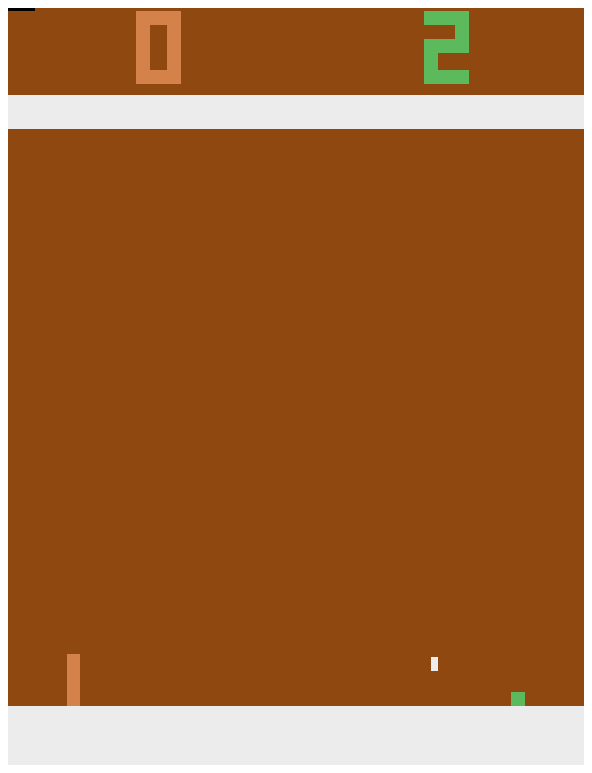}
\centering
   \end{minipage}
\begin{minipage}{0.1\textwidth}
\includegraphics[width=1\textwidth]{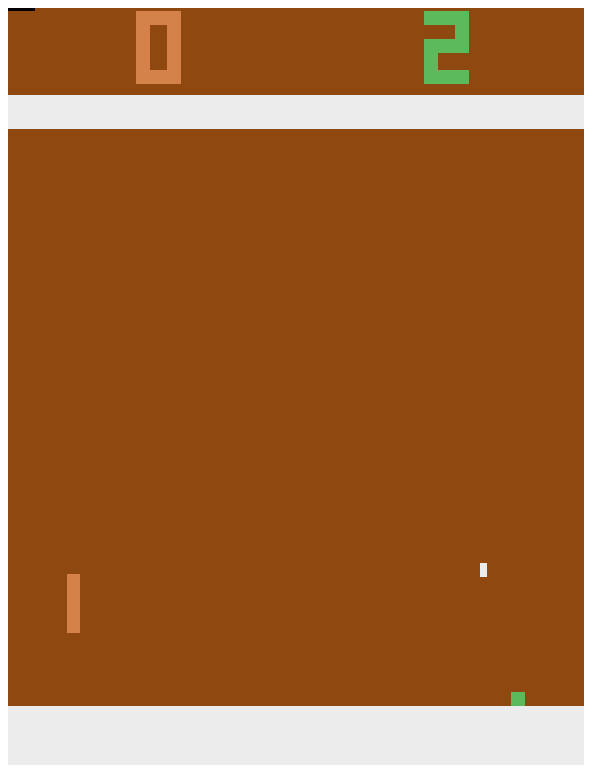}
\centering
\end{minipage}
\begin{minipage}{0.1\textwidth}
\includegraphics[width=1\textwidth]{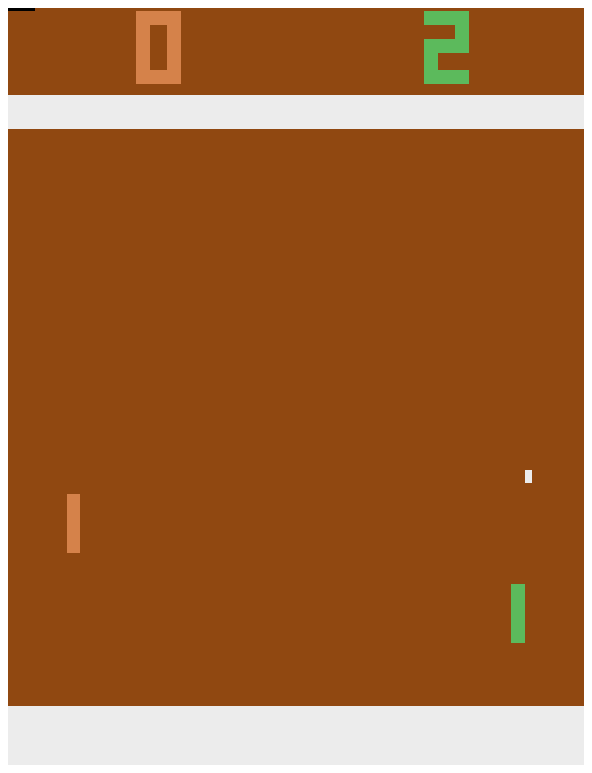}
\centering
\end{minipage}
\begin{minipage}{0.1\textwidth}
\includegraphics[width=1\textwidth]{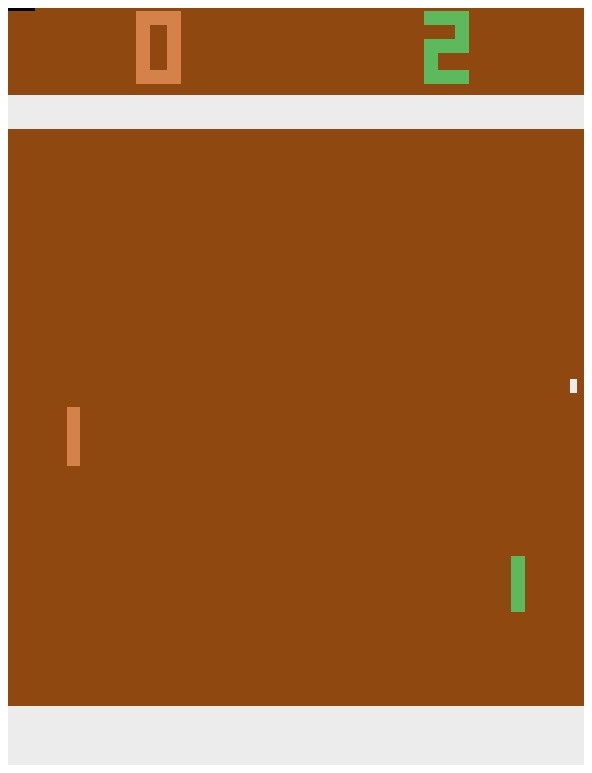}
\centering
\end{minipage}
\hspace*{2.5cm}
\begin{minipage}{0.1\textwidth}
  \centering
  \footnotesize
  \begin{tabular}{l|cccr|rr}
    \multicolumn{1}{c}{}&\multicolumn{3}{c}{Action space} \\
    \toprule
Steps&stay&up&down\\
\midrule
$t$& 4.3918 & 4.3220 & 4.3933\\
$t+1$& 2.7985 & 2.8371 &2.7921\\
$t+2$&2.8089 & 2.8382 &2.8137\\
$t+3$& 3.8725 &3.8795 & 3.8690\\
    \bottomrule
  \end{tabular}
\end{minipage}\hfill 
\caption{The sequence of four consecutive decision states, and corresponding learned Q-function by \DQN at $t,t+1,t+2,t+3$ from left to right, where the agent loses the point. At time step $t$, the optimal action is \textit{up} but the Q-value of going up is lower than other actions. More significantly, even though the agent chooses action \textit{down} and goes down, the Q value of action down at time step $t$ is considerably far from the maximum Q value of the next state at time step $t+1$.}
   \label{fig:pongQ}
\end{figure}
Moreover, in Fig.~\ref{fig:pong-bias} and Table.~\ref{fig:pong-bias} we investigate the case that the agent catches the ball. The ball is going to the right and agent needs to catch it. At time step $t$, the paddle is not on the direction of ball velocity, and as shown in Table.~\ref{fig:pong-bias}, the optimal action is \textit{down}. But a closer look at the estimated Q value of action \textit{up} reveals that the Q value for both action \textit{up} is unreasonably close, when it could lead to losing the point.
\begin{figure}[ht]
\hspace{2.cm}
\vspace*{-0.0cm}
   \begin{minipage}{0.1\textwidth}
\includegraphics[width=1\textwidth]{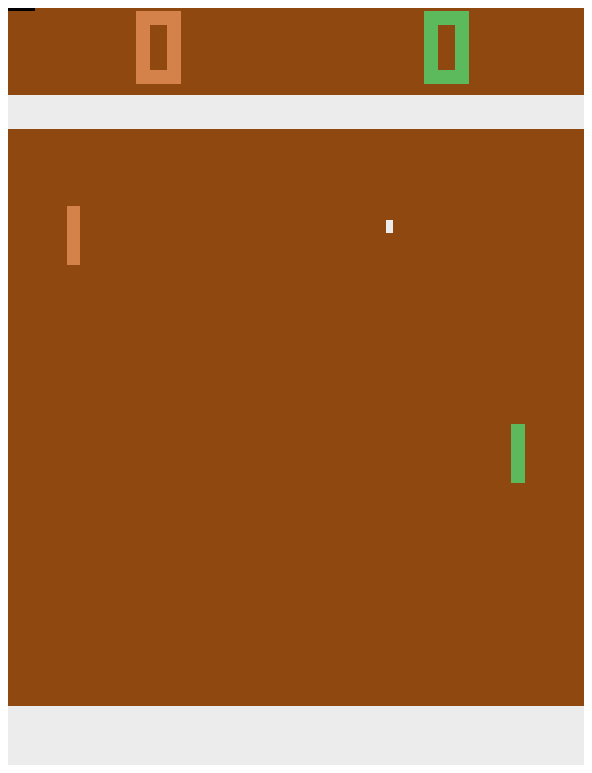}
\centering
 \end{minipage}
\begin{minipage}{0.1\textwidth}
\includegraphics[width=1\textwidth]{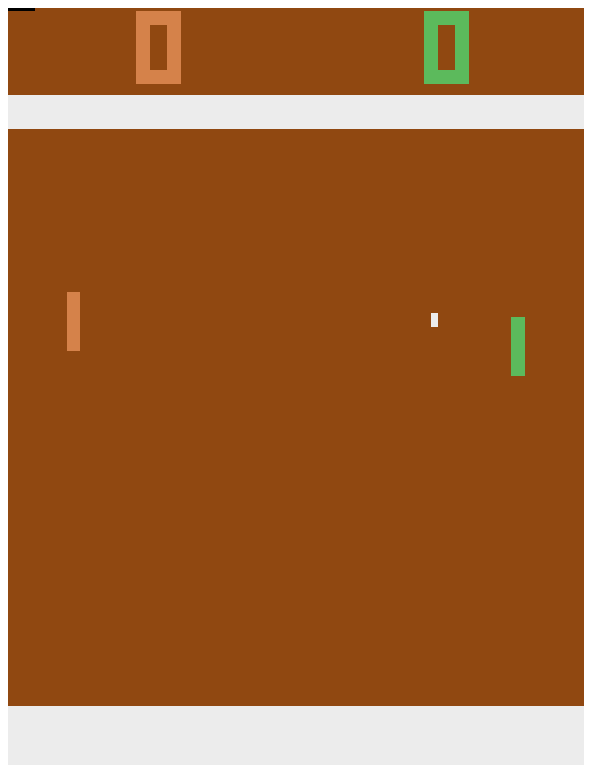}
\centering
\end{minipage}
\hspace*{3.8cm}
\begin{minipage}{0.1\textwidth}
  \centering
  \footnotesize
  \begin{tabular}{l|cccr|rr}
    \multicolumn{1}{c}{}&\multicolumn{3}{c}{Action space} \\
    \toprule
Steps&stay&up&down\\
\midrule
$t$& 1.5546 & 4.5181&   4.5214\\
    \bottomrule
  \end{tabular}
\end{minipage}\hfill
\caption{States at $t-1\rightarrow t$ and the corresponding Q function learned through \DQN at time $t$. Action \textit{up} is sub-optimal but has high value and considerably close to action \textit{down}. While action \textit{down} and \textit{stay} show have more similar values than \textit{up} and \textit{down}}
   \label{fig:pong-bias}
\end{figure}
Lastly, we studied the existing errors in the estimation of the Q function using \DQN. In Table.\ref{fig:pongQ}, if the agent could roll-out even one step before making a decision, it could observe negative consequence of action \textit{down}. The positive effect of the roll-out is more significant in earlier stages of Q learning, where the Q estimation is more off.

\section{\GATS on Pong}\label{apx:pong}

We run \GATS with $1,2,3$, and $4$ steps lookahead ($GATS1,GATS2,GATS3,GATS4$) and after extensive hyper parameter tuning for the \DQN model we show the \GATS performance improvement over \DQN in ~Fig.~\ref{fig:gatsperformance}~\textit{(left)}. 
~Fig.~\ref{fig:gatsperformance}~\textit{(right)} shows the \RP prediction accuracy. We observe that when the transition phase occurs at decision step $1$M, the \RP model mis-classifies the positive rewards. But the \RP rapidly adapts to this shift and reduces the classification error to less than 2 errors per episode.

As DRL methods are data hungry, we can re-use the data to efficiently learn the model dynamics. Fig.~\ref{fig:exp} shows how accurate the \GDM can generate next $9$ frames just conditioning on the first frame and the trajectory of actions. This trajectory is generated at decision step $100k$. Moreover, we extend our study to the case where we change the model dynamics by changing the game mode. In this case, by going from default mode to alternate mode in pong, the opponent paddle gets halved in size. We expected that in this case, where the game became easier, the \DDQN agent would preserve its performance but surprisingly it gave us the most negative score possible, i.e -21 and broke. Therefore, we start fine tuning \DDQN and took 3M time step (12M frame) to master the game again. It is worth noting that it takes \DDQN 5M time step (20M frame) to master from scratch. While \DDQN shows a vulnerable and undesirably breakable behaviour to this scenario, \GDM and \RP thanks to their detailed design, adapt to the new model dynamics in 3k samples, which is amazingly smaller (see more details in \ref{apx:DA})
\begin{figure}[ht]
\vspace*{-0.3cm}
   \begin{minipage}{1\textwidth}
   \centering
\includegraphics[width=0.4\linewidth]{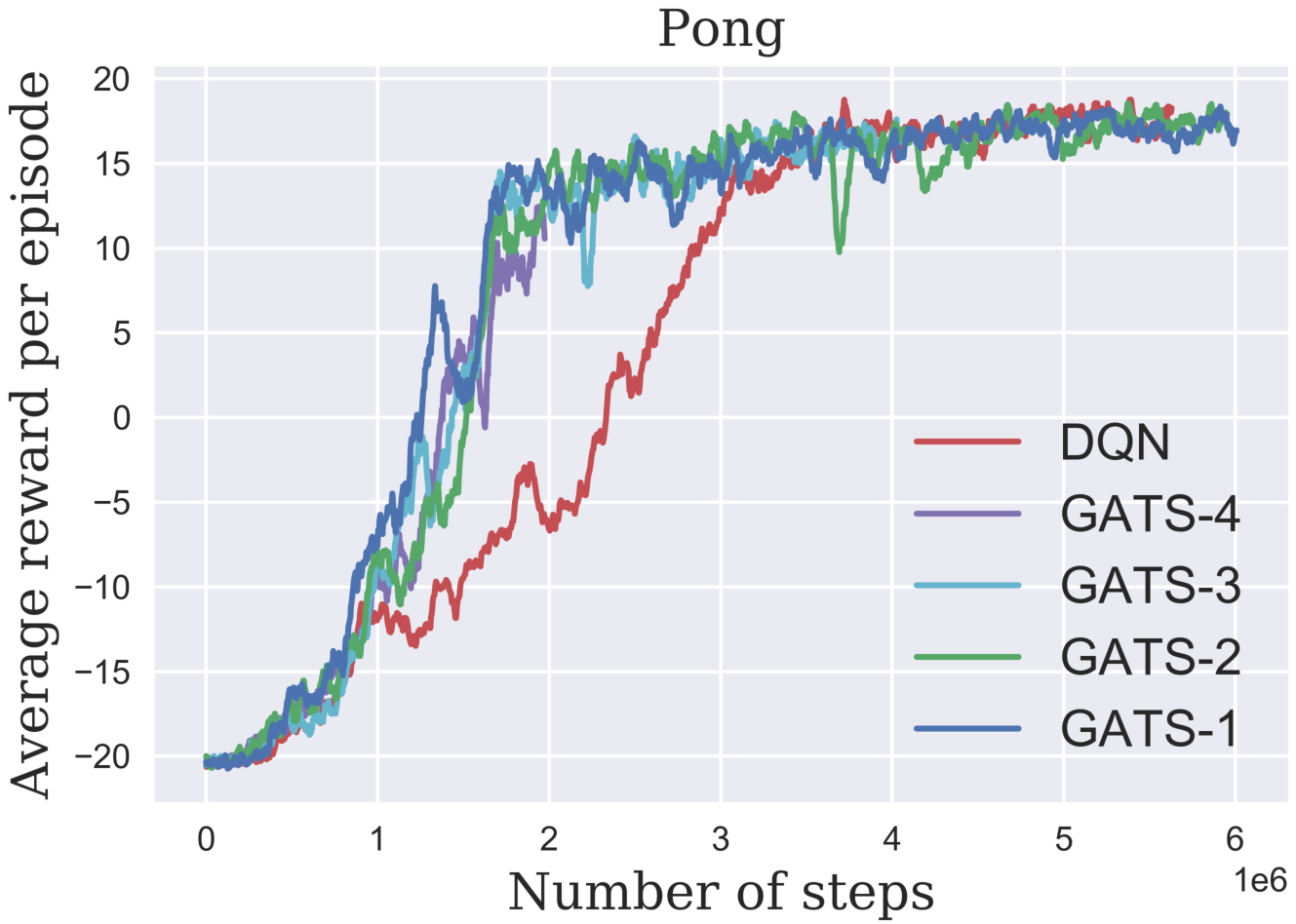}
\includegraphics[width=0.4\linewidth]{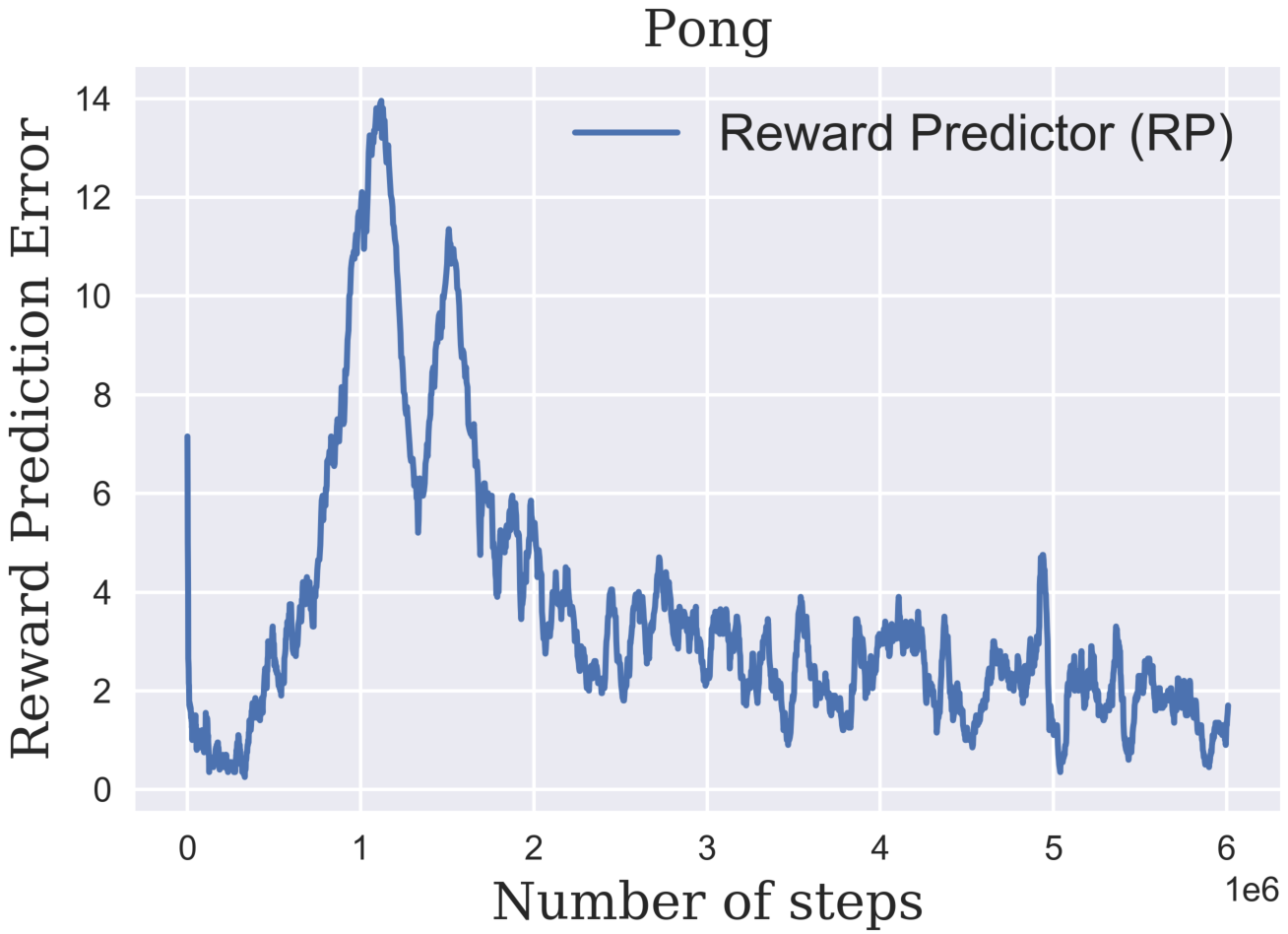}
   \end{minipage} 
\vspace*{-0.2cm}
\caption{\textit{left}:\GATS learns a better policy faster than plain \DQN ($2$ times faster). \GATS$k$ denotes \GATS of depth $k$. \textit{right}: Accuracy of \RP. The $Y$ axis shows the number of mistakes per episode and each episode has average length of $2k$, so the acc is almost always around $99.8\%$. This accuracy is consistent among runs and different lookahead lengths.}

\label{fig:gatsperformance}
\end{figure}
%
%
%

%
In addition to \GATS on \DQN, we also study two other set of experiments on \DDQN. Since Fig.~\ref{fig:gatsperformance} shows that the deeper roll-outs beyond one step do not provide much additional benefit for Pong, we focus on one-step roll-outs for the next two experiments. 
In the first experiment, we equip $\GATS+\DDQN$ with the mentioned Wasserstein-optimism approach, and compare it with \DDQN and plain $\GATS+\DDQN$, which both use $\varepsilon$-greedy based approaches for exploration. In Fig.~\ref{fig:ddqn}\textit{left},~ we observe that this optimism heuristic is helpful for better exploration.

In the second experiment, we investigate the effect of prioritizing training samples for the \GDM, fresher samples are more probable to be chosen, which we do in all experiments reported in Fig.~\ref{fig:ddqn}\textit{left}. We study the case where the input samples to \GDM are instead chosen uniformly at random from the replay buffer in  Fig.~\ref{fig:ddqn}\text{right}. In this case the \GATS learns a better policy faster at the beginning of the game, but the performance stays behind \DDQN, due to the shift in the state distribution. It is worth mentioning that for optimism based exploration, there is no $\varepsilon$-greedy, which is why it gets close to the maximum score of $21$. We tested \DDQN and \GATS-\DDQN with $\varepsilon=0$, and they also perform close to $21$. We further extend the study of \GDM to more games \ref{fig:exp} and observed same robust behaviour as Pong. We also tried to apply \GATS to more games, but were not able to extend it due to mainly its high computation cost. We tried different strategies of storing samples generated by MCTS, e.g. random generated experience, trajectory followed by Q on the tree, storing just leaf nodes, max leaf, also variety of different distributions, e.g. geometric distributing, but again due the height cost of hyper parameter tuning we were not successful to come up with a setting that \GATS works for other games

\begin{figure}[h]
\vspace*{-0.2cm}
   \begin{minipage}{1\textwidth}
   \centering
\includegraphics[width=0.4\linewidth]{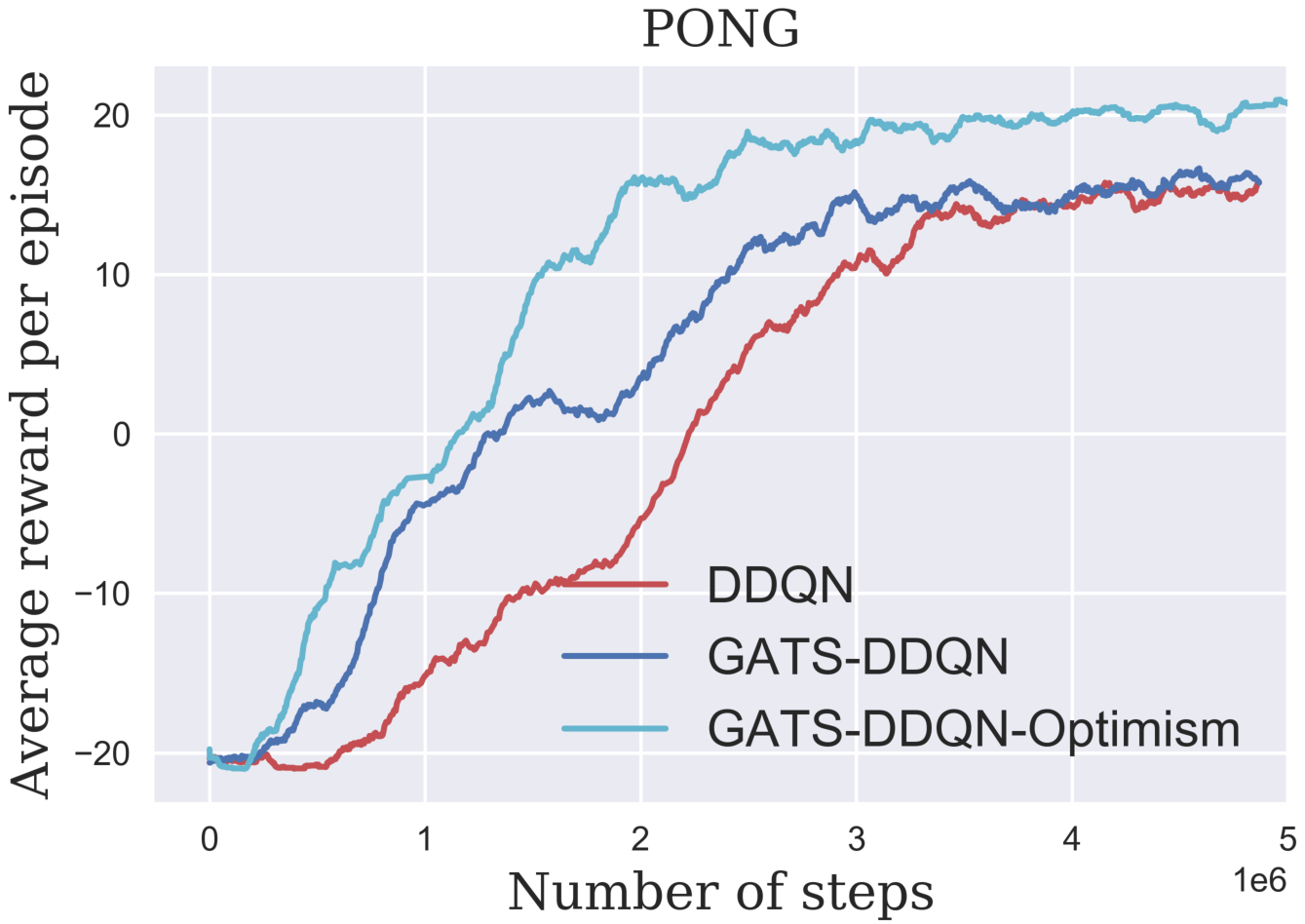}
\includegraphics[width=0.4\linewidth]{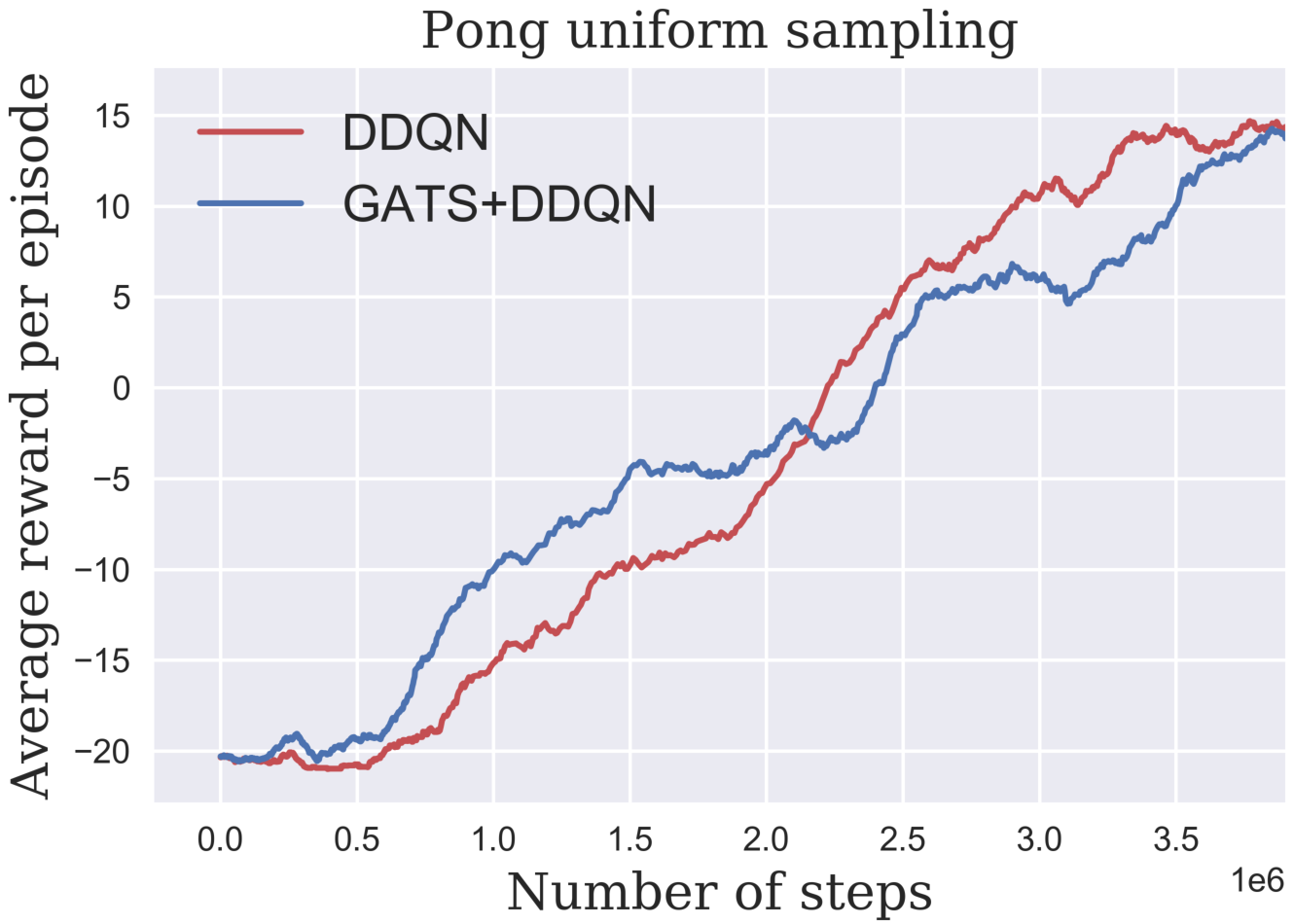}

   \end{minipage} 
\caption{\textit{left}:The optimism approach for \GATS improves the sample complexity and learns a better policy faster. \textit{right}: Sampling the replay buffer uniformly at random to train \GDM, makes \GDM slow to adapt to novel parts of state space.}
\label{fig:ddqn}
\end{figure}

\section{Asterix and Breakout Negative Results}\label{apx:negav}
We include the results for \GATS with 1 step look-ahead (\GATS-1) and compare its performance to DDQN as an example for the negative results we obtained with short roll-outs with the \GATS algorithm. While apply the same hyper parameters we tuned for pong, for Asterix results in performance slightly above random policy, we re-do the hyper-parameter tuning specifically for this game again and Fig.~\ref{fig:Asterix-1} is the best performance we achieved. 

This illustrates the challenges of learning strong global policies with short roll-outs even with near-perfect modeling.

\begin{figure}[ht!]
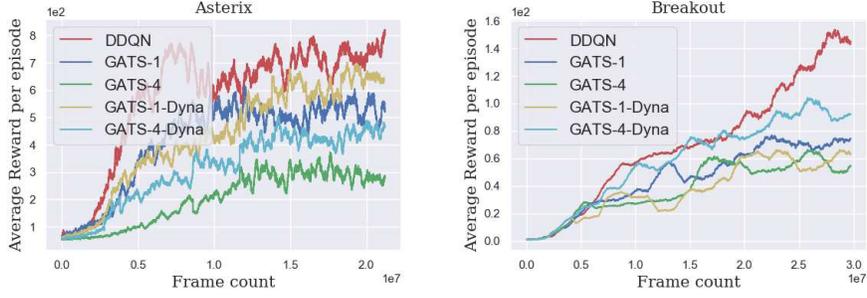

\centering
\includegraphics[scale=0.4]{Fig/Asterix_small.eps}
\includegraphics[scale=0.4]{Fig/Breakout_small.eps}
\caption{The \GATS algorithm on Asterix and Breakout with 1 and 4 step look-ahead, compared to the DDQN baseline.}
\label{fig:Asterix-1}
\end{figure}


\begin{figure}[ht]
\begin{center}
\begin{psfrags}
\psfrag{x0}{$x_t$}
\psfrag{x01}{$\wh{x}_1$}
\psfrag{x02}{$\wh{x}_2$}
\psfrag{x03}{$\wh{x}_3$}
\psfrag{x11}{$\wh{x}_4$}
\psfrag{x12}{$\wh{x}_5$}
\psfrag{x13}{$\wh{x}_6$}
\psfrag{x21}{$\wh{x}_7$}
\psfrag{x22}{$\wh{x}_8$}
\psfrag{x23}{$\wh{x}_9$}
\psfrag{x31}{$\wh{x}_{10}$}
\psfrag{x32}{$\wh{x}_{11}$}
\psfrag{x33}{$\wh{x}_{12}$}
\psfrag{a1}{$a_1$}
\psfrag{a2}{$a_2$}
\psfrag{a3}{$a_3$}
\psfrag{q1}{$Q(\wh{x}_4,a_Q({\wh{x}_4}))$}
\psfrag{q2}{$Q(\wh{x}_5,a_Q(\wh{x}_5))$}
\psfrag{q3}{$Q(\wh{x}_6,a_Q(\wh{x}_6))$}
\psfrag{q4}{$Q(\wh{x}_7,a_Q(\wh{x}_7))$}
\psfrag{q5}{$Q(\wh{x}_8,a_Q(\wh{x}_8))$}
\psfrag{q6}{$Q(\wh{x}_9,a_Q(\wh{x}_9))$}
\psfrag{q7}{$Q(\wh{x}_{10},a_Q(\wh{x}_{10}))$}
\psfrag{q8}{$Q(\wh{x}_{11},a_Q(\wh{x}_{11}))$}
\psfrag{q9}{$Q(\wh{x}_{12},a_Q(\wh{x}_{12}))$}
\includegraphics[width=10cm]{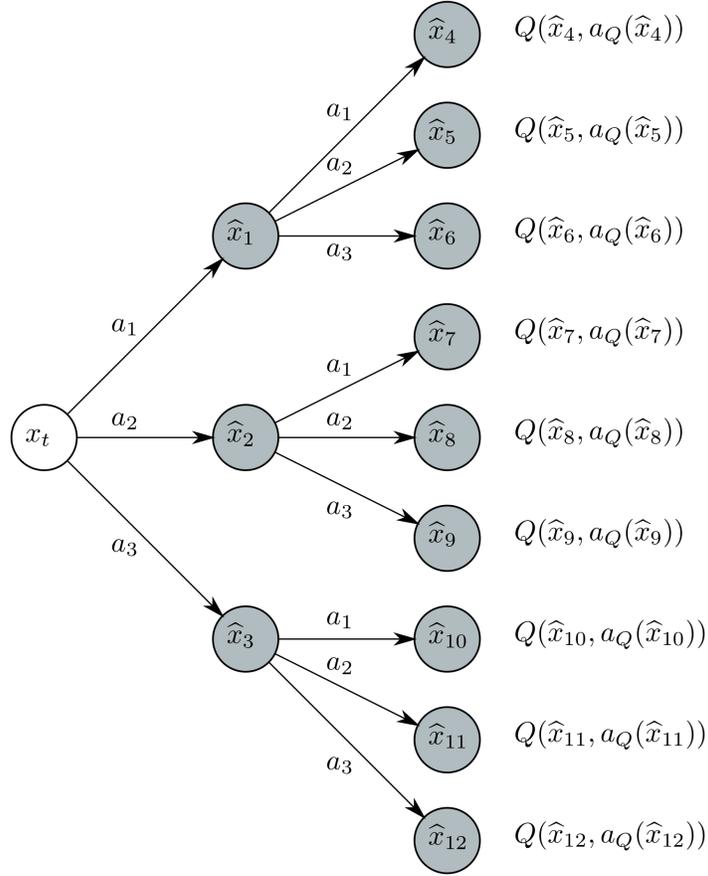}
\end{psfrags}
\end{center}
\caption{Roll-out of depth two starting from  state $x_t$. Here $\wh{x}$'s are the generated states by \GDM. $Q(x,a(x))$ denotes the predicted value of state $x$ choosing the greedy action $a_Q(x):=\arg\max_{a'\in\A}Q(x,a') $.}
  \label{fig:mcts}
\end{figure}

\begin{figure}[ht]
\begin{center}
\begin{psfrags}
\psfrag{x0}{$x_t$}
\psfrag{xt}{${x}_{1+1}$}
\psfrag{xtt}{${x}_{t+2}$}
\psfrag{x02}{$\wh{x}_2$}
\psfrag{x21}{$\wh{x}_7$}
\psfrag{a1}{$a_1$}
\psfrag{a2}{$a_2$}
\psfrag{a3}{$a_3$}
\includegraphics[width=6cm]{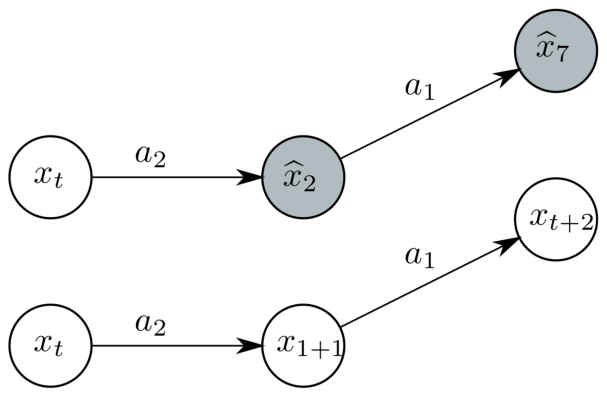}
\end{psfrags}
\end{center}
  \caption{Training GAN and $Q_{\theta'}$ using the longer trajectory of experiences}
  \label{fig:mcts-planning}
  \vspace*{-0.4cm}
\end{figure}

\section{\GDM Architecture and parameters}\label{apx:GDM}

For the generative dynamic model, 
we propose a generic \GDM which consists of a generator $G$ 
and a discriminator $D$, 
trained adversarially w.r.t. the extended conditional Wasserstein metric 
between two probability measures $\Prob_\varpi,\Prob_G$ 
under a third probability measure $\Prob$;
\begin{align}\label{eq:wasser}
W(\Prob_\varpi,\Prob_G;\Prob):=&\sup\!_{D\in \|\cdot\|_L}\E_{\varpi\sim \Prob_\varpi|\varrho, \varrho\sim \Prob}[D(\varpi|\varrho)]-\E_{\varpi: G(\varrho\sim \Prob,z\sim\N(0,I))}[D(\varpi|\varrho)]
\end{align}
Here, $z$ is a mean-zero unit-variance Gaussian random vector
and $\|\cdot\|_L$ indicates the space of $1$-Lipschitz functions. 
In \GDM, $D$ solves the interior $\sup$, 
while $G$'s objective is to minimize this distance 
and learn the $\Prob_\varpi|\varrho$ for all $\varrho$. 
In \GATS, $\Prob$ 
is the distribution over pairs of $\varrho:(x,a)$ in the replay buffer, 
and $\Prob_\varpi|\varrho$ is the distribution 
over the successor states $\varpi:x'$, 
as the transition kernel $T(x'|x,a)$.

The \GDM model consists of seven convolution and also seven deconvolution layers. Each convolution layer is followed by Batch Normalization layers and the leaky ReLU activation function with negative slope of $-0.2$. Also each deconvolution layer is followed by a Batch Normalization layer and the ReLU activation instead of leaky RELU. The encoder part of the network uses channel dimensions of $32,32,64,128,256,512,512$ and kernel sizes of $4,4,4,4,2,2,2$. The reverse is true for the decoder part. We concatenate the bottleneck and next 5 deconvolution layers with a random Gaussian noise of dimension 100, the action sequence, and also the corresponding layer in the encoder. The last layer of decoder is not concatenated. Fig.~\ref{fig:gdm_generator}. For the discriminator, instead of convolution, we use SN-convolution~\citep{miyato2018spectral} which ensures the Lipschitz constant of the discriminator is below 1. The discriminator consists of four SN-convolution layers followed by Batch Normalization layers and a leaky RELU activation with negative slope of $-0.2$. The number of channels increase as $64, 128, 256,16$ with kernel size of $8,4,4,3$, which is followed by two fully connected layers of size $400$ and $18$ where their inputs are concatenated with the action sequence. The output is a single number without any non-linearity. The action sequence uses one hot encoding representation. 

We train the generator using Adam optimizer with weight decay of $0.001$, learning rate of $0.0001$ and also $beta1,beta2 = 0.5,0.999$. For the discriminator, we use SGD optimizer with smaller learning rate of $0.00001$, momentum of $0.9$, and weight decay of $0.1$. Given the fact that we use Wasserstein metric for \GDM training, the followings are the generator and discriminator gradient updates: for a given  set of 5 frames and a action, sampled from the replay buffer, $(f_1,f_2,f_3,f_4,a_4,f_5)$ and a random Gaussian vector $z$:

Discriminator update:
\begin{align*}
\nabla_{\theta_D}\left[\frac{1}{m}\sum_{i=1}^mD_{\theta_D}(f_5,f_4,f_3,f_2,a_4)-\frac{1}{m}\sum_{i=1}^mD_{\theta_D}(G_{\theta_G}(f_4,f_3,f_2,f_1,a_4),f_4,f_3,f_2,a_4,z)\right]
\end{align*}

Generator update:
\begin{align*}
\nabla_{\theta_G}\left[-\frac{1}{m}\sum_{i=1}^mD_{\theta_D}(G_{\theta_G}(f_4,f_3,f_2,f_1,a_4,z),f_4,f_3,f_2,a_4)\right]
\end{align*}
where $\theta^{\GDM}=\lbrace\theta_G,\theta_D\rbrace$ are the generator parameters and discriminator parameters. In order to improve the quality of the generated frames, it is common to also add a class of multiple losses and capture different frequency aspects of the frames~\cite{isola2017image,oh2015action}. Therefore, we also add $10*L1+90*L2$ loss to the GAN loss in order to improve the training process. It is worth noting twh these losses are defined on the frames with pixel values in $[-1,1]$, therefore they are small but still able to help speed up the the learning. In order to be able to roll-out for a longer and preserve the \GDM quality, we also train the generator using self generated samples, i.e. given the sequence $(f_1,f_2,f_3,f_4,a_4,f_5,a_5,f_6,a_6,f_7,a_7,f_8)$, we also train the generator and discriminator on the generated samples of generator condition on its own generated samples for depth of three. This allows us to roll out for longer horizon of more than 10 and still preserve the \GDM accuracy. 

\begin{figure}[ht!]
\centering
\includegraphics[scale=0.3]{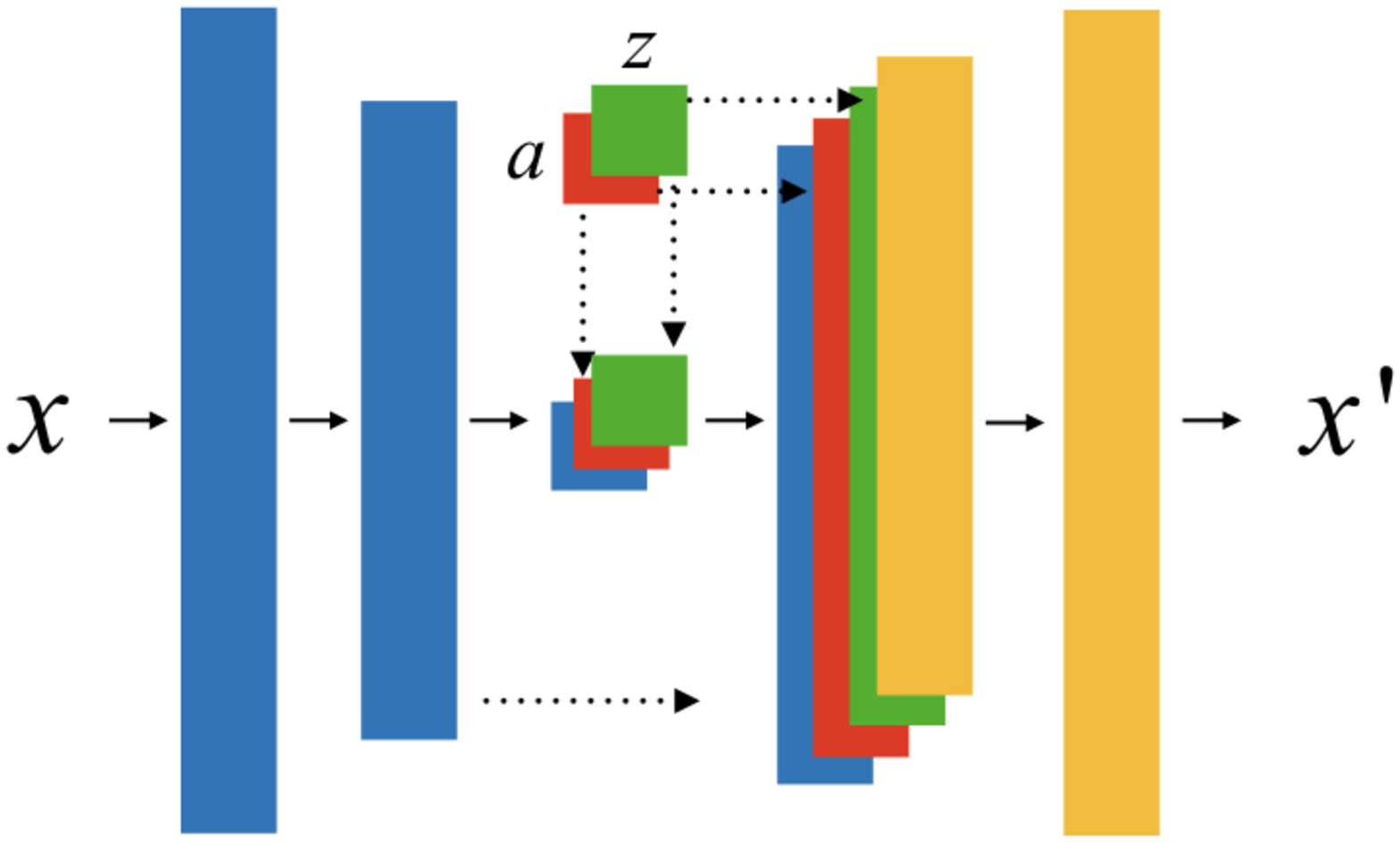}
\caption{The GDM generator is an encoder-decoder architecture with skip-connections between mirrored layers, with action and Gaussian noise concatenated in the bottleneck and decoder layers.}
\label{fig:gdm_generator}
\end{figure}

\begin{figure}
\centering
\hspace*{-1.1cm}
\begin{tabular}{cc}
\includegraphics[scale=0.5]{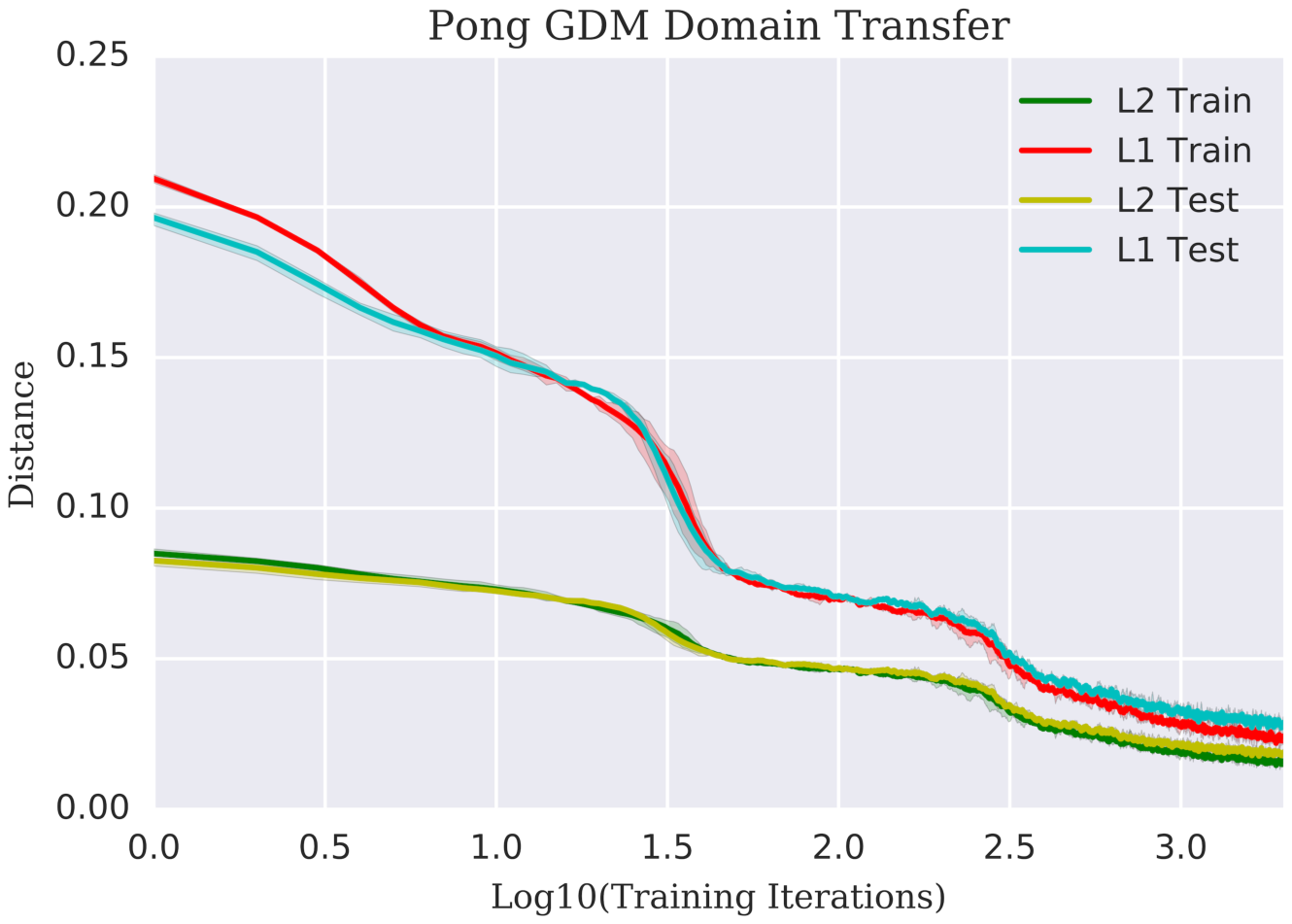}
&
\includegraphics[scale=0.5]{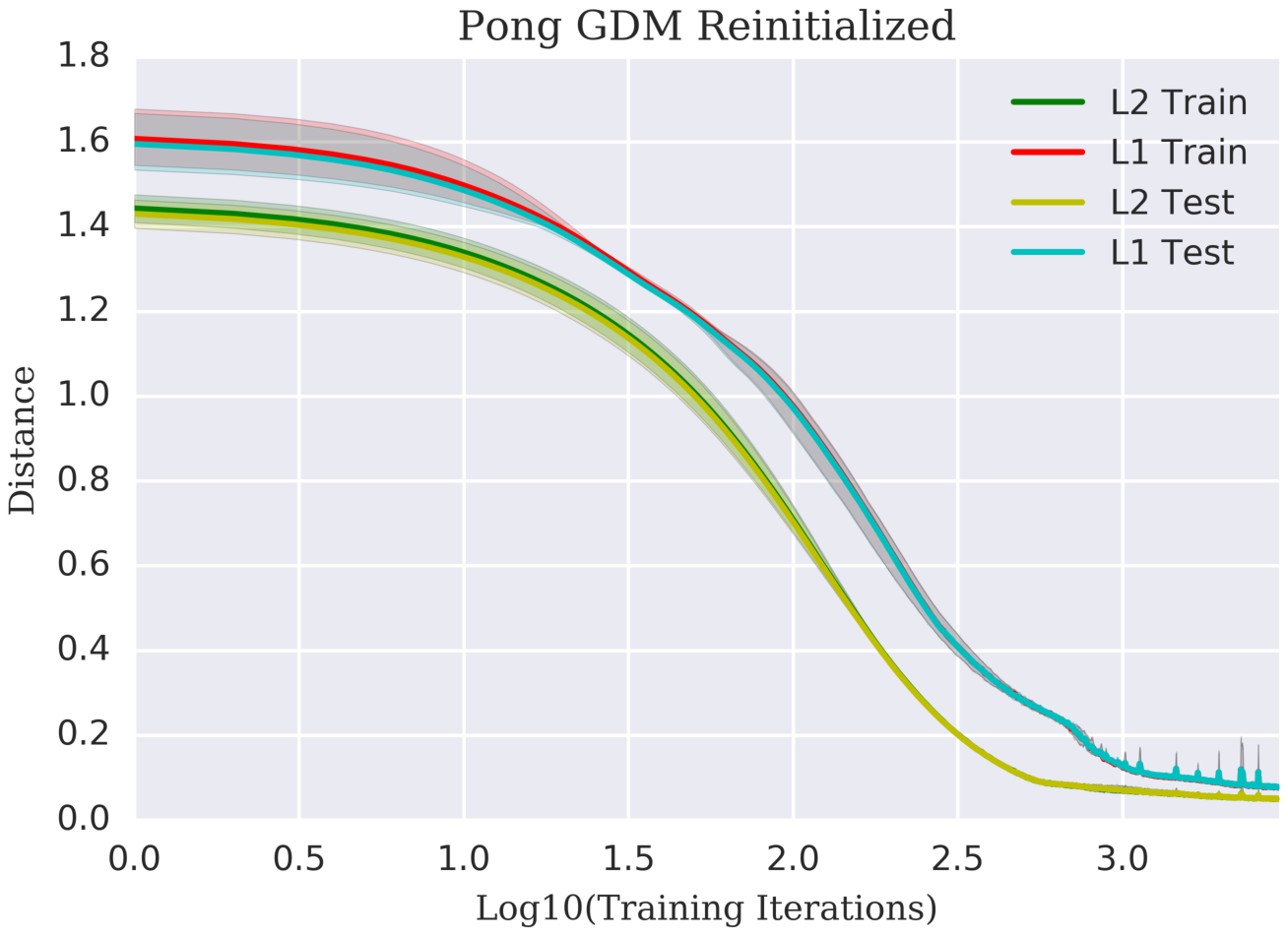}
\\
\multicolumn{2}{c}{\includegraphics[scale=0.06]{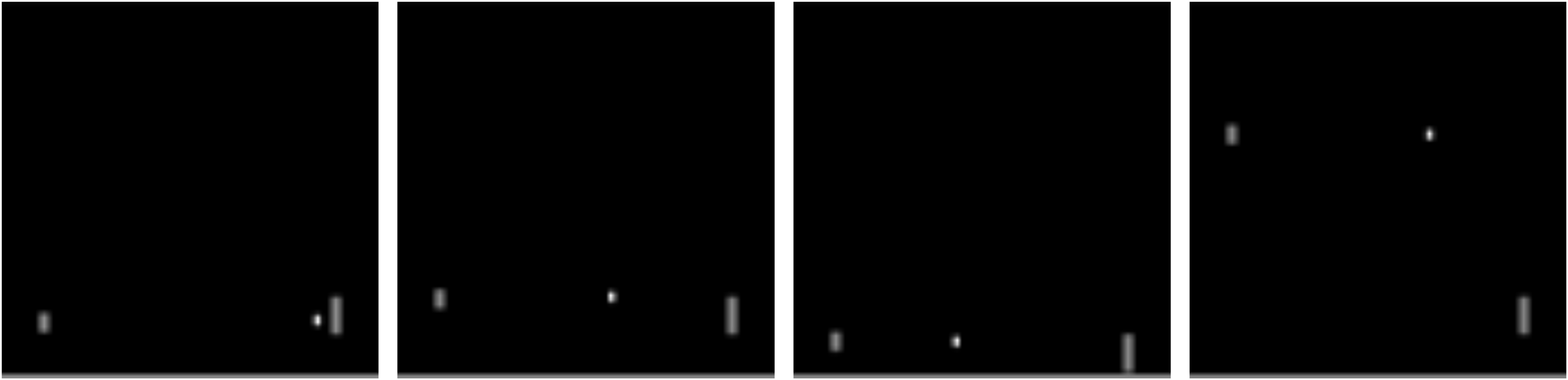} }
\\
\includegraphics[scale=0.06]{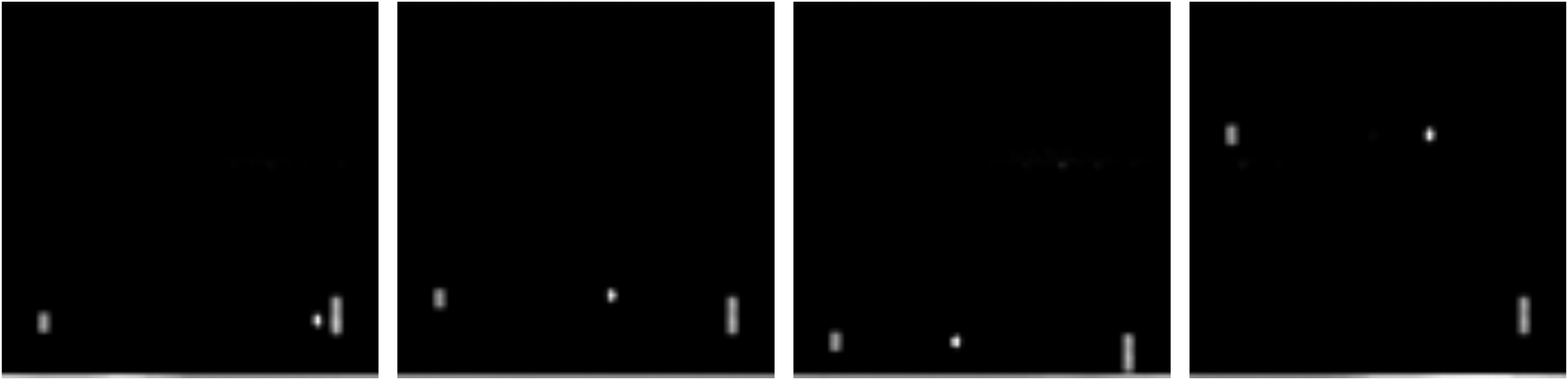} 
&
\includegraphics[scale=0.06]{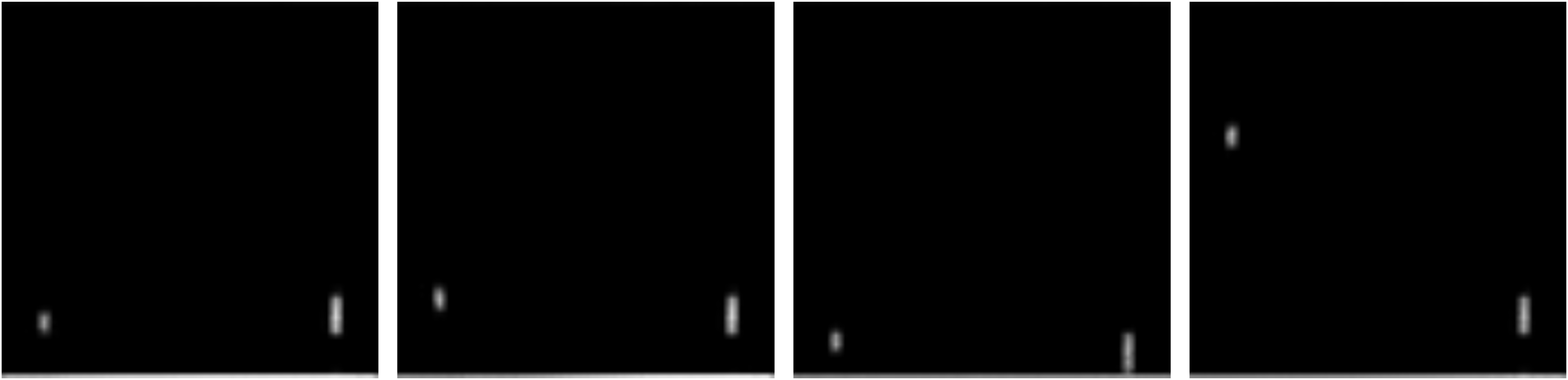}
\end{tabular}
\caption{Training and evaluating domain transfer for \GDM on new game dynamics for Pong (Mode 1, Difficulty 1). \GDM domain transfer from Pong (Mode 0, Difficulty 0) on \textit{left} and \GDM from re-initialized parameters on \textit{right}. L1 and L2 loss curves displayed \textit{top}. Ground truth next frames displayed \textit{middle} with predicted next frames displayed \textit{bottom}.}
\label{fig:gdm_curves}
\end{figure}

\paragraph*{Q function on generated frames} Ideally, if the \GDM model is perfect at generating frames 
i.e. the space generated frames is, pixel by pixel, the same as the real frames,  
for the leaf nodes $x_H$, 
we can use $\max_a Q(x_H,a;\theta)$, learned by the \DQN model on real frames, 
in order to assign values to the leaf nodes. 
But in practice, instead of $x_H$,
we have access to $\wh{x}_H$, 
a generated state twh perceptually is similar to $x_H$ (Fig.~\ref{fig:exp}),
but from the perspective of $Q_{\theta}$, 
they might not be similar over the course of training of $Q_{\theta}$. In order to compensate for this error, we train another Q-network, 
parameterized with $\theta'$, 
in order to provide the similar Q-value as $Q_{\theta}$ 
for generated frames.
To train $Q_{\theta'}$, we minimize the $L2$ norm between $Q_{\theta'}$ and $Q_{\theta}$ for a given GAN sample state and trajectory Fig.~\ref{fig:mcts-planning}. For this minimization, we use Adam with learning rate of $0.0001$, no weight decay, and $beta1,beta2 = 0.5,0.999$. We experimented with weight decay and adding $L1$ loss, but we find these optimizations degrade the performance of the network. We tracked the difference between $Q_{\theta}(\wh{x})-Q_{\theta}(x)$ and $Q_{\theta'}(\wh{x})-Q_{\theta}(x)$ and observed twh both of these quantities are negligible. We ran \GATS without the $Q{\theta'}$, with just $Q_{\theta}$, and observed only slightly worse performance.

\subsection{\GDM Domain Adaptation.}\label{apx:DA} We evaluate the \GDM's ability to perform domain adaptation using the environment mode and difficulty settings in the latest Arcade Learning Environment~\citep{machado2017revisiting}. We first fully train \GDM and \DDQN on Pong with Difficulty 0 and Mode 0. We then sample 10,000 frames for training the \GDM on Pong with Difficulty 1 and Mode 1, which has a smaller paddle and different game dynamics. We also collect 10,000 additional frames for testing the \GDM. We train \GDM using transferred weights and reinitialized weights on the new environment samples and observe the L1 and L2 loss on training and test samples over approximately 3,000 training iterations, and we observe twh they decrease together without significant over-fitting in Fig.~\ref{fig:gdm_curves}. To qualitatively evaluate these frames, we plot the next frame predictions of four test images in  Fig.~\ref{fig:gdm_curves}. We observe twh training \GDM from scratch converges to a similarly low L1 and L2 loss quickly, but it fails to capture the game dynamics of the ball. This indicates the L1 and L2 loss are bad measurements of a model's ability to capture game dynamics. \GDM is very efficient at transfer. It quickly learns the new model dynamics and is able to generalize to new test states with an order of magnitude fewer samples than the Q-learner.

\subsection{\GDM Tree Rollouts}\label{apx:losses}
Finally, we evaluate the ability of the \GDM to generate different future trajectories from an initial state. We sample an initial test state and random action sequences of length 5. We then unroll the \GDM for 5 steps from the initial state. We visualize the different rollout trajectories in Figs. \ref{fig:gdm_trees}\ref{fig:gdm_trees1}.

\begin{figure}[ht!]
\centering
\includegraphics[scale=.52]{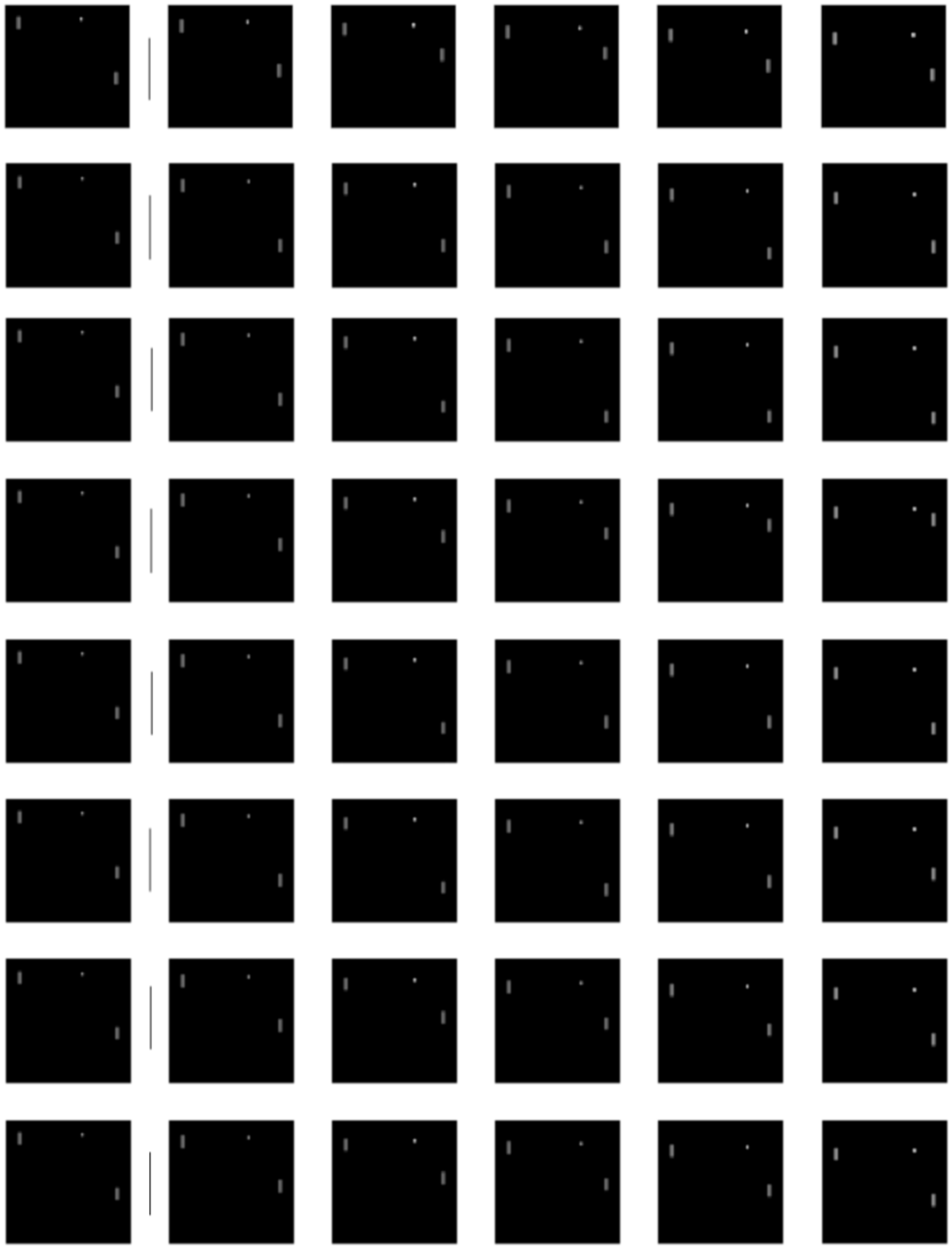}
\caption{Eight 5-step roll-outs of the \GDM on the Pong domain. Generated by sampling an initial state with 8 different 5-action length sequences.}
\label{fig:gdm_trees}
\end{figure}

\begin{figure}[ht!]
\centering
\includegraphics[scale=.52]{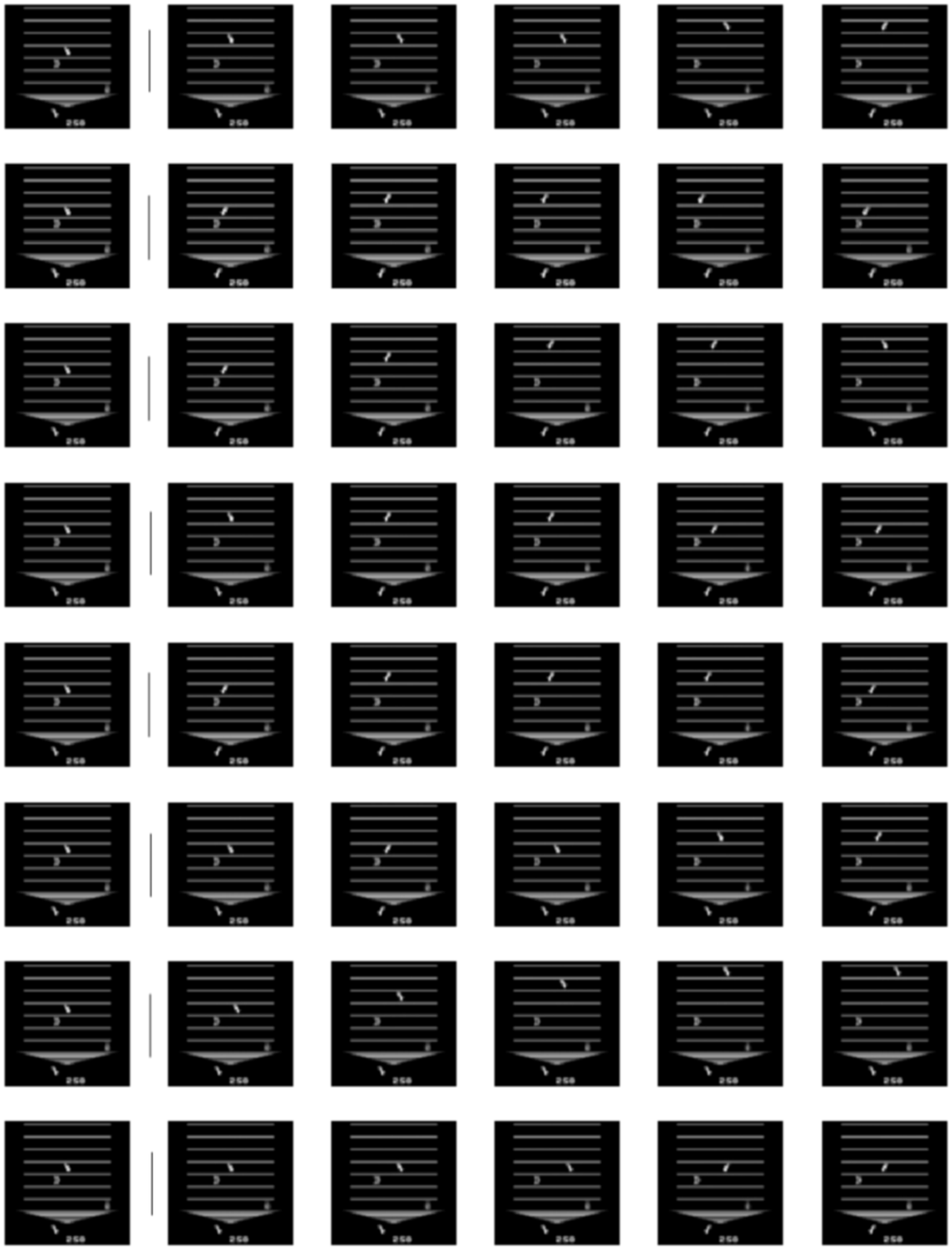}
\caption{Eight 5-step roll-outs of the \GDM on the Asterix domain. Generated by sampling an initial state with 8 different 5-action length sequences.}
\label{fig:gdm_trees1}
\end{figure}

\end{document}